\documentclass{article}

\usepackage[preprint]{neurips_2023}
\usepackage{adjustbox}

\usepackage{subcaption}

\PassOptionsToPackage{numbers, compress}{natbib}

\usepackage[utf8]{inputenc} 
\usepackage[T1]{fontenc}    
\usepackage{xcolor}         
\definecolor{navyblue}{rgb}{0, 0, 0.50196} 
\usepackage[colorlinks=true, allcolors=navyblue]{hyperref}       
\usepackage{url}            
\usepackage{booktabs}       
\usepackage{multirow}

\usepackage{amsfonts}       
\usepackage{nicefrac}       
\usepackage{microtype}      
\usepackage{tikz}
\usepackage{soul}
\usepackage[inline]{enumitem}   
\usepackage{caption} 
\usepackage{amssymb}
\usepackage{xcolor}

\usepackage{amsthm}
\usepackage{dsfont}

\definecolor{C0}{HTML}{1F77B4}
\definecolor{C1}{HTML}{FF7F0E}
\definecolor{C2}{HTML}{2CA02C}
\definecolor{C3}{HTML}{D62728}
\newcommand{\yesSymbol}{\textcolor{C2}{\checkmark}}
\newcommand{\noSymbol}{\textcolor{C3}{$\times$}}

\usepackage{soul, color}
\newcommand{\Note}[1]{}
\renewcommand{\Note}[1]{#1}  

\usepackage{amsmath,amsfonts,bm}

\def\1{\bm{1}}

\def\d{\mathrm{d}}

\def\fX{{\bm{X}}}
\def\bX{{\bm{X}}}

\def\Pm{\mathcal{A}}
 
\def\m{{\vy}} 
\def\A{{\mA}}

\def\bH{{\bm{H}}}

\def\bW{{\bm{W}}}

\def\ddf{\overrightarrow{\mathrm{d}}}

\def\rv{{\textnormal{v}}}

\def\dd{{\mathrm{d}}}

\def\vh{{\bm{h}}}

\def\vx{{\bm{x}}}
\def\vy{{\bm{y}}}

\def\mA{{\bm{A}}}

\def\mH{{\bm{H}}}

\def\mX{{\bm{X}}}
\def\mY{{\bm{Y}}}

\DeclareMathAlphabet{\mathsfit}{\encodingdefault}{\sfdefault}{m}{sl}
\SetMathAlphabet{\mathsfit}{bold}{\encodingdefault}{\sfdefault}{bx}{n}

\def\gN{{\mathcal{N}}}
\def\gO{{\mathcal{O}}}

\def\sR{{\mathbb{R}}}

\def\erf{{\text{erf}}}

\newcommand{\E}{\mathbb{E}}

\newcommand{\R}{\mathbb{R}}

\let\log\relax
\DeclareMathOperator{\log}{ln}

\def\P{{\mathbb{P}}}

\def\Q{{\mathbb{Q}}}

\def\R{{\mathbb{R}}}

\newcommand{\KL}{D_{\mathrm{KL}}}

\DeclareMathOperator*{\argmin}{arg\,min}

\definecolor{pearDark}{HTML}{2980B9}
\newcommand\cbox[1]{\colorbox{pearDark!20}{$#1$}}
\newcommand\cboxnew[1]{#1}
\usepackage{algorithm}
\usepackage[noEnd=true, indLines=true, spaceRequire=false, italicComments=false, rightComments=true, commentColor=gray]{algpseudocodex}

\renewcommand{\l}[1]{\ensuremath{{#1}_\theta}} 
\renewcommand{\c}[1]{\ensuremath{\mathbf{#1}}}    
\newcommand{\e}[1]{\ensuremath{\hat{#1}}}  

\newcommand{\nn}{\ensuremath{\l{\c{f}}}}  
\newcommand{\noise}{\ensuremath{\bm{\varepsilon}}}  

\newcommand{\dist}{\ensuremath{\mathcal{P}}}  
\newcommand{\distO}{\ensuremath{\dist_0}}  
\newcommand{\distT}{\ensuremath{\dist_T}}  
\newcommand{\distnoise}{\dist_\text{noise}}  
\newcommand{\distsampling}{\dist_\text{sampling}}  
\newcommand{\distdata}{\dist_\text{data}}  

\renewcommand{\rv}[1]{\ensuremath{\MakeUppercase{\c{#1}}}}  

\def\mx{{\bm{x}}}
\def\my{{\bm{y}}}

\newcommand{\fwd}[1]{%
  \tikz[baseline=(char.base)]{
    \node[inner sep=0pt, outer sep=0pt] (char) {$#1$};
    \draw[line width=0.2pt] ($(char.north west)+(0.05em,0.25em)$) -- ($(char.north east)+(0em,0.25em)$);
    \draw[line width=0.2pt] ($(char.north east)+(0em,0.25em)$) -- ($(char.north east)+(-0.15em,0.15em)$);
  }%
} 

\newcommand{\bwd}[1]{%
  \tikz[baseline=(char.base)]{
    \node[inner sep=0pt, outer sep=0pt] (char) {$#1$};
    \draw[line width=0.2pt] ($(char.north west)+(0em,0.25em)$) -- ($(char.north east)+(-0.05em,0.25em)$);
    \draw[line width=0.2pt] ($(char.north west)+(0em,0.25em)$) -- ($(char.north west)+(0.15em,0.15em)$);
  }%
} 
\newcommand{\law}[1]{\mathrm{Law}\left(#1\right)}

\colorlet{shadecolor}{gray!20}
\sethlcolor{shadecolor}

\newcommand{\new}{\textcolor{blue}{(new)}}

\usepackage{todonotes}

\usepackage[capitalize,noabbrev]{cleveref}
\Crefname{algorithm}{Alg.}{Algs.}
\Crefname{equation}{Eq.}{Eqs.}
\Crefname{figure}{Fig.}{Figs.}
\Crefname{tabular}{Tab.}{Tabs.}
\Crefname{section}{Sec.}{Secs.}
\Crefname{proposition}{Prop.}{Props.}
\Crefname{appendix}{App.}{Apps.}
\Crefname{corollary}{Cor.}{Cors.}

\newtheorem{theorem}{Theorem}[section]
\newtheorem{proposition}[theorem]{Proposition}

\newtheorem{corollary}[theorem]{Corollary}

 \makeatletter
\newtheorem*{rep@theorem}{\rep@title}
\newcommand{\newreptheorem}[2]{%
\newenvironment{rep#1}[1]{%
 \def\rep@title{#2 \ref{##1}}%
 \begin{rep@theorem}}%
 {\end{rep@theorem}}}
\makeatother

\newreptheorem{proposition}{Proposition}
\newreptheorem{corollary}{Corollary}
\newreptheorem{theorem}{Theorem}
\newreptheorem{lemma}{Lemma}
\newreptheorem{observation}{Observation}
\newreptheorem{remark}{Remark}

\usepackage[utf8]{inputenc} 
\usepackage[T1]{fontenc}    
\usepackage{hyperref}       
\usepackage{url}            
\usepackage{booktabs}       
\usepackage{amsfonts}       
\usepackage{nicefrac}       
\usepackage{microtype}      
\usepackage{xcolor}         
\usepackage[framemethod=tikz]{mdframed}
\mdfdefinestyle{highlightedBox}{
    linecolor=gray!30,          
    backgroundcolor=gray!10,   
    innertopmargin=5pt,       
    innerbottommargin=2pt,    
    roundcorner=4pt            
    innerleftmargin=0pt,
    innerrightmargin=3pt,
    rightmargin=-0.2cm,
    leftmargin=-0.2cm
}

\newif\ifcontrol

\controlfalse

\newif\iffine

\finefalse

\title{A framework for conditional diffusion modelling with applications in protein design and inverse problems}

\author{%
  Kieran Didi*\\
  University of Cambridge \\
  \texttt{ked48@cam.ac.uk} \thanks{equal contributions}\\
  \And
  Francisco Vargas*\\
  University of Cambridge \\
  \texttt{fav25@cam.ac.uk} \\
  \AND
  Simon Mathis*\\
  University of Cambridge \\
  \texttt{svm34@cam.ac.uk} \\
  \And
  Vincent Dutordoir*\\
  University of Cambridge \\
  \texttt{vd309@cam.ac.uk} \\
  \And
  Emile Mathieu*\\
  University of Cambridge \\
  \texttt{ebm32@cam.ac.uk} \\
    \And
  Urszula Julia Komorowska \\
 University of Cambridge \\
  \texttt{ujk21@cam.ac.uk} \\
    \And
  Pietro Lio \\
  University of Cambridge \\
  \texttt{pl219@cam.ac.uk} \\
}

\begin{document}

\maketitle

\begin{abstract}
Many protein design applications, such as binder or enzyme design, require scaffolding a structural motif with high precision. Generative modelling paradigms based on denoising diffusion processes emerged as a leading candidate to address this \emph{motif scaffolding} problem and have shown early experimental success in some cases. 
In the diffusion paradigm, motif scaffolding is treated as a conditional generation task, and several conditional generation protocols were proposed or imported from the Computer Vision literature. 
However, most of these protocols are motivated heuristically, e.g. via analogies to Langevin dynamics,
and lack a unifying framework, obscuring connections between the different approaches.
In this work, we unify conditional training and conditional sampling procedures under one common framework based on the mathematically well-understood \emph{Doob's h-transform}. This new perspective allows us to draw connections between existing methods and propose a new variation on existing conditional training protocols. We illustrate the effectiveness of this new protocol in both, \emph{image outpainting} and \emph{motif scaffolding} and find that it outperforms standard methods. 
\end{abstract}

\section{Introduction}
Denoising diffusion models are a powerful class of generative models where noise is gradually added to data samples until 
they converge to pure noise.
The time-reversal of this noising process then allows to transform noise into samples. This process has been widely successful in generating high-quality images \citep{ho2020denoising} and has more recently shown promise in designing protein backbones that were validated in experimental protein design workflows \citep{watson2023novo}.

Importantly for protein design, diffusion models allow to subject this time-reversed sampling process to a target condition. For proteins, a key condition is the inclusion of a structural motif that grants the protein a particular function, such as binding to a specific target or forming an enzyme active site. However, for these motifs to be foldable and stable, they often need to be integrated into a larger protein structure. While there have been notable successes in scaffolding some motifs experimentally, many still prove challenging to scaffold \citep{watson2023novo}. This makes the development of better conditional generation methods for diffusion models an active area of research, with several contributions from the computer vision, molecular and protein design communities in recent times.


For instance, several methods cast the conditional sampling problem as an inverse (posterior sampling) problem and propose adding a \emph{guidance} term to the time-reversal's drift (Fig.~\ref{fig:conditioning_overview}c) \citep[e.g.][]{ho2022Video,chung2022diffusion}.
Another line of work, focusing on `inpainting', suggests \emph{replacing} the observed variable in the diffused state~(Fig.~\ref{fig:conditioning_overview}b)~\citep[e.g.][]{song2021Scorebased,dutordoir2023neural,mathieu2023geometric}. Yet other work performs heuristic conditional training with the target variables in place \citep{watson2023novo,torge2023diffhopp}.

In this work, we reinterpret the conditioning problem leveraging Doob's $h$-transform. This new perspective provides theoretical backing to existing approaches and naturally leads us to propose a novel method, which we call \emph{amortised training} (Fig.~\ref{fig:conditioning_overview}d, \cref{algo:cond_dobsh}).
We highlight the straightforward implementation and practical use of our theoretical framework by applying it to problems, first as a proof of concept in image generation. We then study the merits and shortcomings of our newly proposed \emph{amortised training} method in more detail for the \emph{motif scaffolding} problem in protein design. We do so by comparing an amortised training implementation of the small-scale diffusion model Genie \citep{lin2023generating} on the RFDiffusion benchmark as well as a newly proposed benchmark dataset based on the SCOPe classification \citep{chandonia2022scope}.

Our main contributions are as follows:
\begin{enumerate}[label=\itshape\roman*\upshape)]
\item We derive a formal framework for conditioning diffusion processes using Doob's $h$-transform (\cref{htrans}).
\item We use our framework to create a taxonomy of existing methods (\cref{tab:summary}).
\item Our taxonomy elucidated the absence of a specific method within the current literature, prompting us to develop and implement this novel approach (\cref{sec:amortised}).
\item We empirically assess these different approaches on image generation and protein design (\cref{sec:exp}).
\item Finally, we present plug-and-play algorithms to implement various conditioning schemes (\cref{app:algorithms}).
\end{enumerate}

\begin{table}[t]
\small
\centering
\begin{adjustbox}{width=\textwidth,center}
\begin{tabular}{llcccl}
\toprule
{\scshape \textbf{Method}} & {\scshape \textbf{Stage}} & \scshape {\textbf{Operator}}  & \multicolumn{2}{c}{{\scshape \textbf{Constraint}}} &  {\scshape \textbf{Framework}} \\
& & Leveraged & Soft & Hard \\ 
\midrule
Amortised $h$-transform (ours) & Training & \yesSymbol & \yesSymbol & \yesSymbol & Amortised trained $h$  \\
Classifier free \citep{ho2022classifier} & Training & \noSymbol  & \noSymbol & \yesSymbol & Amortised trained $h$\\
Replacement \citep{song2020score} & Sampling & \yesSymbol & \noSymbol & \yesSymbol & ? \\
$\quad$w/ particles: SMCDiff~\citep{trippe2022diffusion} & Sampling &  \yesSymbol & \yesSymbol & \yesSymbol & ? \\
RFDiffusion \citep{watson2023novo} & Training & \yesSymbol & \noSymbol & \yesSymbol & Marginal of $h$ \\
Classifier guidance \citep{dhariwal2021diffusion} & Finetuning &  \noSymbol & \noSymbol & \yesSymbol & Trained separate $p(\m|\bH_t)$ \\
Reconstruction guidance \citep{chung2022diffusion,chung2022improving} & Sampling & \yesSymbol & \yesSymbol & \yesSymbol & Moment matching $h$\\
$\quad$w/ particles: TDS \citep{wu2023practical} & Sampling  & \yesSymbol & \yesSymbol & \yesSymbol & Moment matching $h$\\

\bottomrule
\end{tabular}
\end{adjustbox}
\caption{
Taxonomy of conditional methods.
{\scshape \textbf{Stage}} indicates when the conditional information is acquired.
{\scshape \textbf{Operator}} indicates whether the measurement operator $\Pm$ is assumed to be known and thus leveraged by methods.
{\scshape \textbf{Constraint}} classifies the likelihood as either \emph{hard} or \emph{soft}, as detailed in the main text.
{\scshape \textbf{Framework}} specifies the mechanism by which conditioning is accomplished. The `?' means that it is unclear how this method fits into the $h$-transform framework.
}

\label{tab:summary}
\end{table}

\begin{figure}
    \centering
    \includegraphics[width=\textwidth]{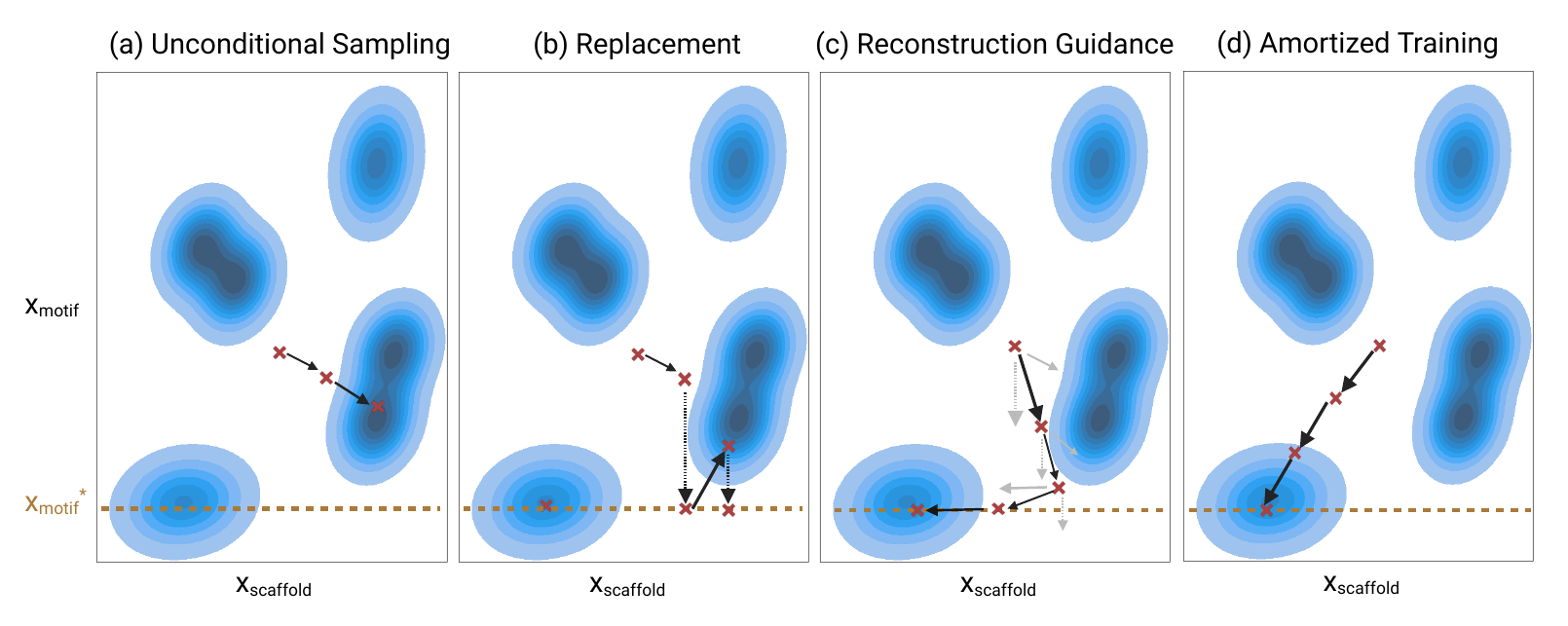}
    \caption{Schematic illustration of several common approaches to (conditionally) sample from a diffusion model. The sampling space is partitioned into motif coordinates (vertical) and scaffold coordinates (horizontal). The target motif is marked as $x_\text{motif}^\star$ and regions with plausible scaffolds are illustrated as blue blobs. A clear definition of each approach as pseudo-algorithm is given in \cref{app:algorithms}.} 
    \label{fig:conditioning_overview}
\end{figure}

\section{Theory: Conditioning diffusions via the $h$-transform, a new perspective}
\label{htrans}

We first show how Doob's $h$-transform enables diffusion models to satisfy \emph{hard} equality constraints and then generalise this result to handle \emph{soft} constraints in the context of noisy observations.

\subsection{Doob's $h$-transform with hard constraint}
\label{sec:hard_guidance}
Doob's transform provides a formal mechanism for conditioning a stochastic differential equation (SDE) to hit an event at a given time.
Formally:

\begin{mdframed}[style=highlightedBox]
\begin{proposition}\label{prop:htrans}(Doob's $h$-transform \cite{rogers2000diffusions})
Consider the reverse SDE:
\begin{align}
\dd \bX_t &= \bwd{b}_t(\bX_t) \,\dd t + \sigma_t \bwd{ \dd \rv{W}}_t, \quad \bX_T \sim  \distT
\label{eq:back_sde},
\end{align}
where time flows backwards and with transition densities  $\bwd{p}_{t|s}$.
It then follows that the conditioned process $\mX_t | \mX_0 \in B$ is a solution of
\begin{align}
\dd \bH_t &= \left(\bwd{b}_t(\bH_t) - \sigma_t^2 \cbox{\nabla_{\mH_t}\ln \bwd{P}_{0|t}(\mX_0 \in B | \mH_t) }\right) \,\dd t + \sigma_t \bwd{ \dd \rv{W}}_t, \quad \bX_T \sim  \distT \label{eq:back_sde_h},
\end{align}
such that $\law{\bH_s | \bH_t} = \fwd{p}_{s|t, 0}(\vh_s| \vh_t, \vx_0 \in B)$ and $\mathbb{P}(\bH_0 \in B) = 1$. 
\end{proposition}
\end{mdframed}
\vspace{-0.2cm}
\vspace{-0.2cm}
%
This says that by conditioning a diffusion process to hit a particular event $\mX_0 \in B$ at a boundary time (e.g.\ $t=0$), the resulting conditional process is itself an SDE with an $\cboxnew{\text{additional drift term}}$. 
Furthermore, the resulting SDE will hit the specified event within a finite time $T$.
The function $\cboxnew{h(t, \bH_t) \triangleq \bwd{P}_{0|t}(\mX_0 \in B\mid \bH_t)}$ is referred to as the \emph{$h$-transform} \citep{rogers2000diffusions,de2021simulating}.
%
The $h$-transform drift decomposes into two terms via Bayes rule, a conditional and a prior score:
\begin{equation}
\cboxnew{\nabla_{\bH_t} \ln  \bwd{P}_{0|t}(\mX_0 \in B \mid \bH_t) = \nabla_{\bH_t} \ln  \fwd{P}_{t|0}(\bH_t \mid \mX_0 \in B) - \nabla_{\bH_t} \ln  {P}_{t}(\bH_t)},
\end{equation}
whereby the conditional score ensures that the event is hit at the specified boundary time, while the prior score ensures it is time-reversal of the correct forward process \citep{de2021simulating} (see \cref{app:intuit}).



\paragraph{Hard constraint}
%
%
We now consider events of the form $\mX_0 \in B$ which are described by an equality constraint $\Pm(\bX_0) = \my$ with $\Pm$ a known \emph{measurement} (or \emph{forward}) operator and $\my$ an observation.
We will see concrete examples of $\Pm$ in \cref{sec:exp}.
\begin{mdframed}[style=highlightedBox]
\begin{corollary}\label{col:inpaint}
Consider the reverse SDE \eqref{eq:back_sde}, 
then it follows that
\begin{align}
\label{eq:rev_sde_htrans}
    \dd \bH_t &= ( \bwd{b}_t(\bH_t) - \sigma_t^2 \cboxnew{\nabla_{\mH_t}\ln \bwd{P}_{0|t}(\Pm(\mX_0) = \m \mid \mH_t) })\,\dd t + \sigma_t \bwd{ \dd \rv{W}}_t, 
\end{align}
satisfies $\law{\bH_s | \bH_t} = \law{ \bX_s | \bX_t, \Pm (\bX_0) =\m}$ thus $\law{  \bH_0}  = \law {\bX_0 | \Pm (\bX_0) =\m}$.
\end{corollary}
\end{mdframed}
\vspace{-0.2cm}
Sampling \eqref{eq:rev_sde_htrans} directly provides samples $\mx \sim p_{\mathrm{data}}$ which also satisfy $\Pm(\mx) = \my$.
Crucially, this SDE is guaranteed to hit the conditioning in finite time, unlike prior equilibrium-motivated approaches \citep{chung2022diffusion,meng2022diffusion,finzi2023user, song2022pseudoinverse,han2022card, dutordoir2023neural}. 


\paragraph{Reconstruction guidance}
%
To get better insight into the challenge of sampling via Doob's $h$-transform \eqref{eq:rev_sde_htrans} let us re-express the $h$-transform as
\begin{align}
    \cboxnew{\bwd{P}_{0|t}(\Pm(\mX_0) = \m \mid \bH_t)} = \int \mathds{1}_{\Pm(\vx_0) = \m} (\vx_0)\bwd{p}_{0| t}(\vx_0| \bH_t) \dd \vx_0
\end{align}
where in the case of denoising diffusion models $\bwd{p}_{0| t}(\vx_0| \cdot)$ is the transition density of the reverse SDE \eqref{eq:back_sde}.
In practice, one does not have access to this transition density -- i.e. we can sample from this distribution, but we cannot easily get its value at a certain point. This makes it difficult to approximate the integral.
To alleviate this, a strand of recent works \citep{finzi2023user,song2022pseudoinverse,rozet2023score} have proposed Gaussian approximation of $\bwd{p}_{0| t}(\vx_0| \cdot) \approx \gN(\vx_0 ~\vert~ \E[\fX_0 |\fX_t = \cdot], \Gamma_t)$ leveraging Tweedie's formula and the already trained score network.
This line of work is referred as \emph{reconstruction guidance}.
We note that whilst proposing to approximate the quantity $\cboxnew{\bwd{P}_{0|t}(\Pm(\mX_0) = \m| \cdot)}$,
they do not make the connection to Doob's transform and thus are unable to provide guarantees on the conditional sampling that \cref{col:inpaint} provides.
Overall, the Gaussian-based approximations of Doob's $h$-transform lead to reconstruction guidance-based approaches \citep{finzi2023user, rozet2023score,chung2022diffusion,han2022card,song2022pseudoinverse} $\dd \bH_t = \left( \bwd{b}_t(\bH_t) +\sigma_t^2 \nabla_{\bH_t} ||\m - \mathrm{A} \E[\fX_0 |\fX_t = \bH_t]  ||^2_{\Gamma_t}\right)\dd t + \sigma_t \bwd{ \dd \rv{W}}_t, $ $\bX_T \sim  \dist_T$, where $\Gamma_t$ acts as a guidance scale \citep{mathis2023nmd, rozet2023score}, and $\mathrm{A}$ is a matrix if $\Pm$ is linear otherwise $\mathrm{A} = \mathrm{d} \Pm(\E[\fX_0 |\fX_t = \bH_t])$.

\subsection{Generalised $h$-transform for soft constraints}
\label{sec:soft_guidance}
In the previous \cref{sec:hard_guidance}, we showed how the $h$-transform allows for conditioning on \emph{hard} constraints, correcting the reverse process to satisfy some observation $P(\my|\mx_0) \propto \mathds{1}_{\Pm(\mX_0) = \m}(\mx_0)$.
Yet, many scenarios deal with \emph{soft} constraints, modelling noisy observation $\my = \Pm(\mx) + \eta$ with a density $p(\my|\mx_0)$, typically 
with the goal of sampling from the posterior $p(\vx_0 | \mY = \vy) = p(\vy | \vx_0) p_{\mathrm{data}}(\vx_0) / p(\vy)$ as in noisy inverse problems \citep{song2021solving,chung2022diffusion,chung2022improving}.
In this section, we present a generalisation of the $h$-transform applicable to denoising diffusion models that build on results in \citep{vargas2021bayesian}:

\begin{mdframed}[style=highlightedBox]
\begin{proposition}(Noisy conditioning)
\label{prop:noisy}
Given the following forward SDE:
\begin{align} \label{eq:vpsde}
    \dd \bX_t &= f_t(\bX_t) \,\dd t + \sigma_t\;\dd\fwd{\bW_t}, \quad \bX_0 \sim  \distdata
\end{align}
it follows that the following reverse SDE with marginals $p_t$
\begin{align} 
    \bH_T &\sim  \law{\bX_T |\bX_0}\nonumber\\
    \dd \bH_t &= \left( f_t(\bH_t) - \sigma^2_t (\nabla_{\bH_t} \ln p_t(\bH_t)+ \nabla_{\bH_t} \ln p_{y|t}(\mY=\m|\bH_t))\right)\,\dd t + \sigma_t\; \bwd{ \dd \rv{W}}_t, \label{eq:rev_sde2trans_2}
\end{align}
satisfies $\law{\bH_0} = p(\vx_0 | \mY=\m)$
where $p_{y|t}(\mY=\m|\cdot) = \int p(\mY=\m|\vx_0)\bwd{p}_{0|t}(\vx_0 |\cdot) \dd \vx_0$.
\end{proposition}
\end{mdframed}
\vspace{-0.2cm}
In short, the above results give a variant of the $h$-transform that allows to sample from noisy posteriors.
This provides theoretical backing to methodologies such as DPS \citep{chung2022diffusion}, in which the SDE  \eqref{eq:rev_sde2trans} is used to solve noisy inverse problems.
%
\begin{mdframed}[style=highlightedBox]
\begin{corollary}
\label{col:noisy_ou}
Furthermore, for an Ornstein-Uhlenbeck (OU) forward process, i.e.\ with drift $f_t(\vx)=-\beta_t\vx$ and diffusion $\sigma_t = \sqrt{2\beta_t}$, we have that 
\begin{align} 
    \dd \bH_t \!=\! \!-\beta_t  \!\left(\bH_t\! +\! 2\nabla_{\bH_t}\! \ln p_t(\bH_t)\!+\! 2\nabla_{\bH_t} \!\ln p_{y|t}(\mY\!=\!\m|\bH_t)\right)\!\dd t \!\!+\! \!\sqrt{2\beta_t }\; \bwd{ \dd \rv{W}}_t, \; \bH_T \! \sim\!\gN(0, I)\label{eq:rev_sde2trans}
\end{align}
satisfies $\law{\bH_0} \approx p(\vx_0 | \mY=\m)$.
As such, $\bH_T$ inherits the rapid convergence guarantees of the OU process~\citep{de2022convergence,de2021diffusion}, in particular
 $|| \law{\bH_T} - \gN(0,I)||_{\mathrm{TV}}\leq \gO \big(e^{-{T}/{\bar{\beta}}}\big)$ for some $\bar{\beta} > 0$.
\end{corollary}
\end{mdframed}

For a careful derivation derivation of \ref{prop:noisy} see Appendix \ref{sec:genh}, prior works apply Bayes rule to the score omitting several steps in between required to make this argument rigorous, whilst the result is simple we believe this is the first work to carefully formalise this.



%
\subsection{Amortised training of $h$-transform}
\label{sec:amortised}

In this section, we propose an objective for learning Doob's $h$-transform at training time in an amortised fashion instead of enforcing the constraint during inference time as before in reconstruction guidance approaches.

Note that since $\cboxnew{\bwd{P}_{0|t}(\Pm(\mX_0) = \m| \mX_t =\vh)} = \fwd{P}_{t|0}(\vh |\Pm(\mX_0) = \m)p_0(\Pm(\mX_0) = \m) / p_t(\mX_t =\vh)$ we can re-express the Doob's transformed SDE of a reversed OU process as:
\begin{align}
    \dd \bH_t &= -\beta_t  \left(\bH_t + 2\nabla_{\bH_t} \ln \fwd{P}_{t|0}(\bH_t \vert \Pm(\mX_0) = \m )\right)\,\dd t + \sqrt{2\beta_t }\; \bwd{ \dd \rv{W}}_t, \;\; \bH_T \sim  \law{\bX_T}. \nonumber
\end{align}
\begin{mdframed}[style=highlightedBox]
\begin{proposition} \label{prop:train}
    The minimiser of
\begin{align}
f^*\! = \!\argmin_f\! \ \mathbb{E}_{\mY \sim p_{|\Pm,\fX_0},\Pm \sim p, \fX_0 \sim p_{\mathrm{data}}}\!\left[\!\int_0^T \!\!\!\!\!||f(t, \fX_t, \mY, \Pm) -\nabla_{\fX_t}\!\ln \fwd{p}_{t|0}(\fX_t|\fX_0)||^2 \dd t\!\right]
\end{align}
is given by the conditional score
    $f^*_t(\vh, \m,  \A) = \nabla_{\vh} \log \fwd{p}_{t|0}(\vh \vert \mY = \m , \Pm = \A).$
\end{proposition}
\end{mdframed}
\vspace{-0.2cm}
This is referred as \emph{amortised} learning for conditional sampling, since practically the neural network approximating the (conditional) score is amortised over $\Pm$ and $\m$, instead of learning a separate network for each condition.
This approach is reminiscent of `classifier free guidance'~\citep{ho2022classifier} where the score network is amortised over some auxiliary variable (e.g.\ as in text-to-image models~\citep{pmlr-v139-ramesh21a}), or of RFDiffusion~\citep{watson2023novo} where proteins are designed given a specific subset motif.
Our framework is different to `classifier free guidance' as $\Pm$ is assumed to be known (e.g.\ an inpainting mask), and to RFDiffusion since the conditioning variable $\mY$ being a subset of $\bX$, is also being noised during training and denoised when sampling (see~\cref{algo:cond_dobsh}). Also note that due to its formulation, classifier guidance would be unable to noise a subset of $\bX$ (the motif) as we do.

\subsection{Conditional Finetuning - Learning the Generalised $h$-transform in Noisy Inverse Problems}
In Appendix \ref{app:stoch_control} we prove that the $h$-transform for the noisy setting can be formulated as the solution to the following optimisation problem:
\begin{align}
    f^* =\argmin_{f} \E_\Q\left[\int_0^T \beta_t|| f(\bH_t) ||^2 \dd t\right] - \E_{\bH_0 \sim \Q_0} [\ln p(\vy| \bH_0)] \label{eq:stoch}
 \end{align}
where $\bH_t$ follows the unconditioned score SDE with an added control $f$ and $f^*_t(\vh) = \nabla_{\vh}\ln \E_{\bX_0 \sim p_{0 |t}(\cdot|\vh)}[p(\vy| \bX_0) ] =\nabla_{\vh}\ln p_{y|t}(\vy |\vh)$ coincides with the generalised $h$-transform. This objective provides a way to learn the conditioned SDE from the unconditioned one, without making Gaussian approximations. In practice, this can be achieved by fine-tuning a pretrained unconditional score model to learn the conditional score (see Appendix \ref{app:stoch_control}). Furthermore, this objective only requires the availability of a single noisy measurement $\vy$.  

We have different options to optimise this finetuning objective:
\begin{itemize}
    \item Naive backpropagation: We can simulate the chain $\Q$ and backpropagate through the full chain. However, this comes with a high memory cost and is infeasible for high-dimensional image datasets. 
    \item Adjoint SDE: We can make use of the adjoint SDE to estimate the gradients 
    \item Use different loss functions: Instead of using this objective, we can instead use VarGrad/TrajectorieBalance objective with the same minimiser. 
\end{itemize}

\section{Experimental results}
\label{sec:exp}
%
To compare the various conditional generation methods, we first highlight our results from initial tests in the image setting and then discuss the motif scaffolding problem in protein design in more detail.

\subsection{Conditional image generation.}
\label{sec:image_generation}
The task of `image outpainting' mimics the motif scaffolding problem in protein design and amounts to conditioning the diffusion model on a central patch of an image. The measurement model $\Pm \in \{0,1\}^{n\times d}$ will select $n$ central pixels out of an image in $\R^d$.
We consider noise-free conditions (i.e.\ hard constraints).
We focus on the {\scshape CelebA}~\citep{liu2015faceattributes} and {\scshape Flowers}~\citep{nilsback08} image datasets.
We empirically evaluate the {\scshape Amortised} approach where the mask is provided at training time as an extra channel, along with {\scshape Reconstruction Guidance} (\cref{algo:uncond_Doob'sh}) and {\scshape Replacement} (\cref{algo:repaint}) methods for which the score network is trained without access to the mask, and are then queried at sampling time.
The quality of conditional samples is measured by the mean squared error (MSE) and LPIPS perceptual metric~\citep{zhang2018perceptual}. See \cref{app:experimental_details_images} for further details.
We empirically observe from \cref{tab:image-outpainting} that the \textsc{amortised} approach slightly outperforms sampling-based methods, which are on par with each other.

\begin{figure}[tbh]
\centering
\hspace*{-2cm}
\begin{minipage}[b]{0.45\textwidth}
\includegraphics[width=0.98\textwidth]{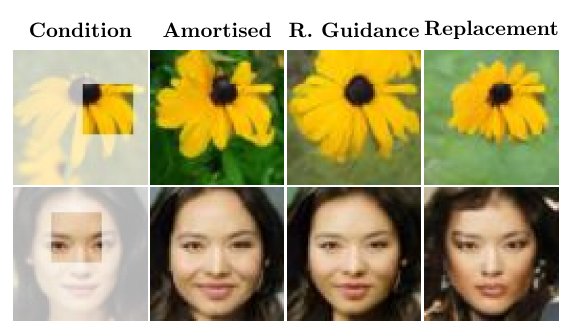}
\vspace{-0.1cm}
\caption{Some conditional samples.}
\end{minipage}
\begin{minipage}[b]{0.45\textwidth}

\adjustbox{max width=1.26\textwidth}{

\centering

\begin{tabular}{lccc}
\toprule
\scshape Metric & \scshape Amortised & \scshape R. Guidance & \scshape Replacement \\
\midrule
\textbf{\scshape Flowers} \\
MSE $(\downarrow)$ & $0.34_{\pm 0.01}$ & $0.27_{\pm 0.01}$ & $0.28_{\pm 0.01}$\\
LPIPS $(\downarrow)$ & $0.25_{\pm 0.00}$ & $0.29_{\pm 0.01}$ & $0.33_{\pm 0.01}$\\
\midrule
\textbf{\scshape CelebA} \\
MSE $(\downarrow)$ & $0.26_{\pm 0.01}$ & $0.30_{\pm 0.01}$ & $0.34_{\pm 0.00}$\\
LPIPS $(\downarrow)$ & $0.14_{\pm 0.00}$ & $0.15_{\pm 0.01}$ & $0.17_{\pm 0.00}$\\
\bottomrule
\end{tabular}
}
\vspace{0.22cm}
\captionof{table}{Quantitative assessment of conditional samples w.r.t\ to ground-truth.}
\label{tab:image-outpainting}
\end{minipage}
\vspace{-0.2cm}
\end{figure}

\subsection{Conditional protein design: motif scaffolding}\label{sec:protein_design} 
The task of motif scaffolding in our protein setting amounts to sampling protein C alpha atom coordinates $\mx \in \R^{d}$ such that it contains a given subset of C alpha coordinates $\my \in \R^{n}$, i.e.\ $\my=\Pm(\mx) = \mathrm{A} \mx$, where $\Pm \in \{0,1\}^{n \times d}$ is a masking matrix which selects $n$ observed C alpha coordinates.
We perform two sets of motif scaffolding experiments. We firstly compare our proposed {\scshape amortised} approach to {\scshape replacement} and {\scshape reconstruction guidance} as we did in the image case. Upon observing that {\scshape amortised} performs significantly better, we then dive into a more detailed analysis of this method on the RFDiffusion benchmark, as well as a new SCOPe-based benchmark that is created from a hierarchical structure and sequence-based split described below.

\begin{figure}[tbh]
\centering
\hspace*{-2cm}
\begin{minipage}[b]{0.45\textwidth}
\includegraphics[width=0.98\textwidth]{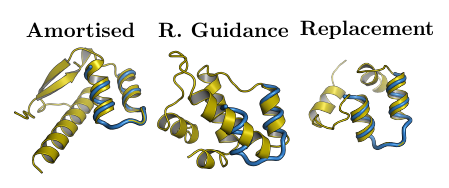}
\vspace{-0.3cm}
\caption{Conditional protein designs in \newline yellow with target motif \texttt{3IXT} in blue.}
\end{minipage}
\hspace*{0.1cm}
\begin{minipage}[b]{0.45\textwidth}

\adjustbox{max width=1.26\textwidth}{
\centering
\begin{tabular}{lccc}
\toprule
\scshape Metric & \scshape Amortised & \scshape R. Guidance & \scshape Replacement \\
\midrule
\% Success $(\uparrow)$ & $20.0$ & $1.5$ & $0.5$\\
\% scRMSD < 2 \AA $(\uparrow)$ & $35.1$ & $45.1$ & $0.6$\\
\% mRMSD < 1 \AA $(\uparrow)$ & $50.0$ & $4.2$ & $24.3$\\

\bottomrule
\end{tabular}
}
\vspace{0.45cm}
\vspace{-0.3cm}
\captionof{table}{{\scshape RFDiff} benchmark metrics (averaged over the 11 targets, 100 samples each). \\Success: pAE < 5, scRMSD < 2\AA, \newline motifRMSD < 1\AA, pLDDT > 70, scTM > 0.5. Details in \cref{sec:protein_design}.}
\label{tab:image-outpainting_prot}
\end{minipage}
\vspace{-0.2cm}
\end{figure}
\label{fig:protein_comp}

\paragraph{Data}\label{app:scope_data}
We evaluate on the RFDiffusion motif benchmark \citep{watson2023novo} and on a self-curated SCOPe benchmark based on a hierarchical structure-based split, jointly with a sequence-similarity-based split.
For the RFDiffusion benchmark, we tested all sequence-contiguous motifs, resulting in 11 different motif design tasks.
Our method readily extends to the non-contiguous motif setting and future work will address this in more detail.
The performance on each of these targets is depicted in~\cref{fig:rfdiff_comp}.
For the SCOPe dataset, we leverage the hierarchical structure classification scheme of the SCOPe database \citep{chandonia2022scope} to create train-test splits that allow us to investigate how well the model can scaffold motifs from unseen folds, families and superfamilies and how difficult these tasks are with respect to each other. In particular, for training, we hold out four clusters of protein structures at the fold level, four at the family level and four at the superfamily level (Fig.~\ref{fig:scope_overview}a) and evaluate the motif-scaffolding performance of the model on this structure-based hold-out set (Fig.~\ref{fig:scope_overview}b-d).

\paragraph{Diffusion process}
We use a discrete-time DDPM~\citep{ho2020denoising} formulation for the diffusion model with $N=1000$ steps and cosine $\beta$-schedule \citep{dhariwal2021diffusion}.

\paragraph{Noise model}
The denoising model $\varepsilon_\theta$ is adapted from the Genie diffusion model \citep{lin2023generating}. In Genie, the denoiser architecture consists of an SE(3)-invariant encoder and an SE(3)-equivariant decoder. While the network uses Frenet-Serret frames as intermediate representations, the diffusion process itself is defined in Euclidean space over the C alpha coordinates. Similar to AlphaFold2, the denoiser network consists of a single representation track that is initialised via a single feature network and a pair representation track that is initialised via a pair feature network. These two representations are further transformed via a pair transform network and are used in the decoder for noise prediction via IPA \cite{jumper2021highly}.\\ To evaluate unconditional sampling-based methods, we retrained the Genie denoising network for 4000 epochs on 4 A100 GPUs ($\sim$300 A100 hours in total). We stopped training at this point, as we observed an almost comparable performance to the publicly available model weights (which were obtained after training for 50'000 epochs).\\To evaluate the {\scshape Amortised} approach (\cref{algo:cond_dobsh}), we perform a minor modification to the unconditional Genie model by adding an additional conditional pair feature network that takes the motif frames as input with the ground truth coordinates for the motif and 0 as values for all other coordinates that are not part of the motif. The output of this motif-conditional pair feature network is concatenated with the output of the unconditional pair feature network to form an intermediate dimension of twice the channel size compared to the unconditional model, before being linearly projected down to the channel size of the unconditional model. From then onward the output is processed by the remaining Genie components as in the unconditional model. The implementation is therefore similar to the image case, where the motif features are presented as additional input and the model learns to use these for reconstructing the motif.
This minor alteration of the Genie architecture means our amortised network has $4.162$M parameters while the unconditional Genie networks have $4.087$M parameters ($\sim$ 1.8\% fewer).

\paragraph{Methods}
In the amortised setting we follow the pseudo-code definition given in~\cref{algo:cond_dobsh}. In 80\% of the training steps, we pass a condition to the network. The other 20\% contains an empty mask consisting of only 0's.
For the reconstruction guidance method (\cref{algo:uncond_Doob'sh}), we use a time-dependent guidance term of $\gamma_t=\alpha_t (1-\alpha_t)$. 

\begin{figure}[htb!]
    \centering
    \includegraphics[width=\textwidth]{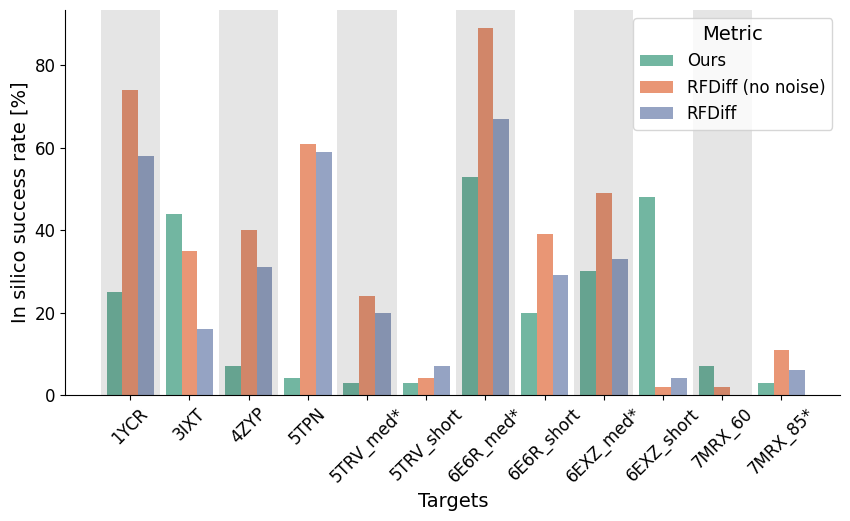}\vspace{-0.3cm}
    \caption{Comparison of our method to RFDiffusion for motif scaffolding for 12 continuous targets. Note that we trained our 4.1M parameter model for only 4000 epochs ($\sim$300 A100 hours in total), which is significantly less both in compute and parameter size than RFDiffusion ($\sim$26'000+ A100 hours, 59.8M parameters). For the motifs marked with *, we had to shorten the sampled scaffold ranges on both sides of the motif from 0-65 (0-63 for TMRX80) to 0-50 since we trained our version of Genie only on protein generation up to a length of 128 residues. Performance numbers from RFDiffusion are taken from the original publication \cite{watson2023novo} and our designs were created with the same design specifications as described there. We note that our folding step uses ESMFold instead of AlphaFold2, but we have future plans to use AlphaFold2 for a more direct comparison.}
    \label{fig:rfdiff_comp}
\end{figure}

\paragraph{Metrics}
We measure the performance of the methods across two axes: designability and success rate. 

To assess whether a particular protein scaffold is \emph{designable}, we run the same pipeline as \cite{lin2023generating}, consisting of an inverse folding generated $C_\alpha$ backbones with ProteinMPNN and then re-folding the designed sequences via ESMFold. The considered metrics and their corresponding thresholds are the following:

\begin{itemize}

    \item scTM > 0.5: This refers to the TM-score between the structure that's been designed and the predicted structure based on self-consistency as previously described. The scTM-score ranges from 0 to 1. Higher scores indicate a higher likelihood that the input structure can be designed. 

    \item scRMSD < 2 \AA~: The scRMSD metric is akin to the scTM metric. However, it uses the RMSD (Root Mean Square Deviation) to measure the difference between the designed and predicted structures, instead of the TM-score. This metric is more stringent than scTM as RMSD, being a local metric, is more sensitive to minor structural variances.

    \item pLDDT > 70 and pAE < 5: Both scTM and scRMSD metrics depend on a structure prediction method like AlphaFold2 or ESMFold to be reliable. Hence, additional confidence metrics such as pLDDT and pAE are employed to ascertain the reliability of the self-consistency metrics.
\end{itemize}

In addition, we want to judge whether the motif scaffolding was successful or not. Therefore, similar to previous work by \cite{watson2023novo}, we calculate the motifRMSD between the predicted design structure and the original input motif and judge samples with < 1 \AA~motifRMSD as a successful motif scaffold.

\paragraph{Results}
We evaluate all three approaches on the continuous motifs from the {\scshape RFDiff}usion motif benchmark \citep{watson2023novo}. For the {\scshape Amortised} approach we retrain the Genie model \citep{lin2023generating} in an amortised fashion (Alg.~{\ref{algo:cond_dobsh}}), while for the {\scshape R. Guidance} and {\scshape Replacement} methods we used the publicly available unconditional model.
We observe that amortised training outperforms the other approaches, especially replacement sampling (Fig.~\ref{fig:protein_comp}).

%




To better understand how well the {\scshape amortised} conditioning approach works, we break down our model performance on the different targets and compare it to the performance of RFDiffusion (Fig.~\ref{fig:rfdiff_comp}). Despite having trained a smaller model with fewer computing resources, we obtained competitive performance on several targets.

We also ablate our model performance w.r.t. structural dissimilarity of the motif compared to the training set via our previously described SCOPe benchmark.
Testing the motif-scaffolding performance of the amortised model on this data, we see that the scaffolding success decreases from fold-over family to superfamily, indicating that scaffolding a motif from a protein that is more different to the training set is harder (Fig.~\ref{fig:scope_overview}b-c).
We also quantitatively observe an anecdotal phenomenon in protein design: while alpha helices are relatively easy to scaffold, domains from other classes have significantly lower success rates (Fig.~\ref{fig:scope_overview}d).
We hope that this benchmark set will help to address these issues in future modelling efforts.

\begin{figure}
    \centering
    \includegraphics[width=\textwidth]{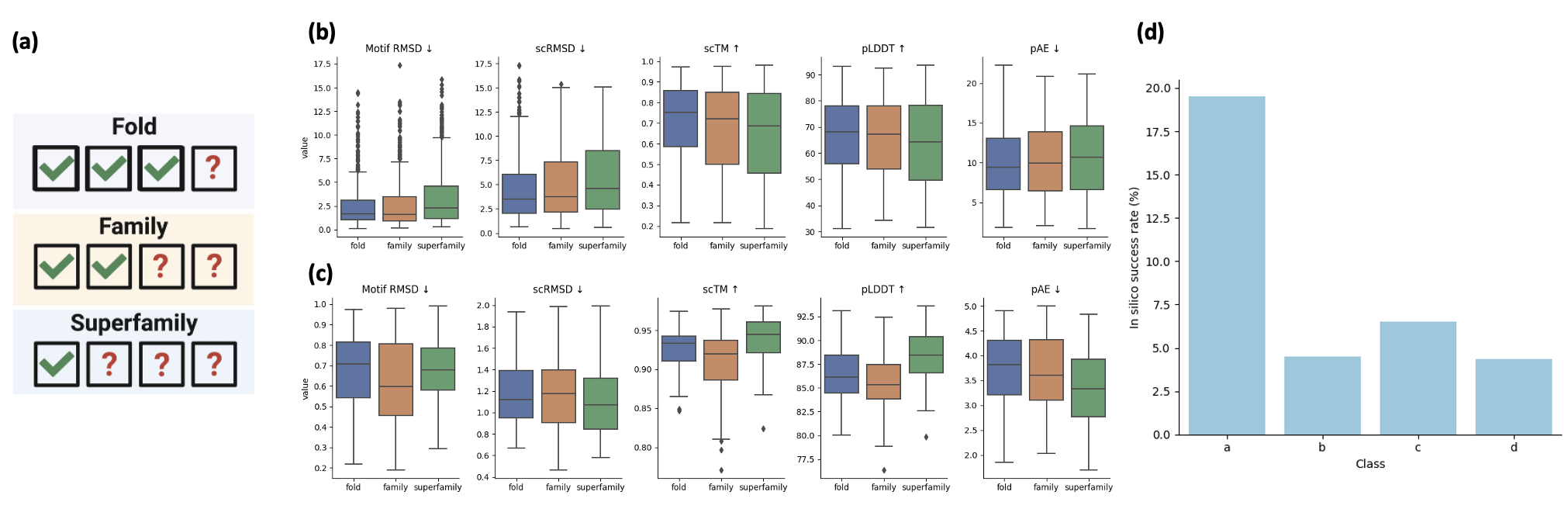}
    \vspace{-0.4cm}
    \caption{Data ablation study on a newly curated SCOPe benchmark dataset with our amortised training model. (a) We utilise the hierarchical structural clustering of SCOPe to create hold-out sets at three different levels of structural hierarchy: the fold, the family and the superfamily level. (b) We test the motif scaffolding performance on these splits and see decreasing scaffolding success for structurally dissimilar samples. (c) The same metrics as in (b), but only for samples that fulfill the definition of in silico success. (d) Scaffolding success by SCOPe class. Alpha helices can be scaffolded successfully, whereas other classes are more challenging.}
    \label{fig:scope_overview}
\end{figure}

\section{Conclusion}

We presented a unified framework, based on Doob's $h$-transform, to better understand and classify different conditional diffusion methods.
Based on the gained insights, we developed a novel {\scshape amortised} conditional sampling scheme (\cref{algo:cond_dobsh}) which differs from existing approaches in that it takes into account the measurement operator. For the motif scaffolding task this means we denoise both the scaffold and the motif.
We evaluated the {\scshape amortised} approach on image outpainting and motif scaffolding in protein design and outperform standard methods. 
We further investigated the performance of the {\scshape amortised} approach by comparing to RFDiffusion on contiguous motifs.
Surprisingly, our {\scshape amortised} implementation of Genie \citep{lin2023generating} achieves notable \emph{in-silico} success rates of between 3 -- 50\% (as per the RFDiffusion definition) across the targets.
Though it lags behind RFDiffusion in 9/12 targets, it achieves this without low-temperature sampling, a mere 10\% of RFDiffusion's parameter count, and after being trained for just 1.2\% of the time. This positions the {\scshape amortised} approach as a promising candidate for further improving motif scaffolding, potentially opening up new applications in protein engineering for drug discovery and enzyme design.

\bibliographystyle{plainnat}
\bibliography{bib}

\begin{thebibliography}{40}
\providecommand{\natexlab}[1]{#1}
\providecommand{\url}[1]{\texttt{#1}}
\expandafter\ifx\csname urlstyle\endcsname\relax
  \providecommand{\doi}[1]{doi: #1}\else
  \providecommand{\doi}{doi: \begingroup \urlstyle{rm}\Url}\fi

\bibitem[Batzolis et~al.(2021)Batzolis, Stanczuk, Sch{\"o}nlieb, and Etmann]{batzolis2021conditional}
Georgios Batzolis, Jan Stanczuk, Carola-Bibiane Sch{\"o}nlieb, and Christian Etmann.
\newblock Conditional image generation with score-based diffusion models.
\newblock \emph{arXiv preprint arXiv:2111.13606}, 2021.

\bibitem[Chandonia et~al.(2022)Chandonia, Guan, Lin, Yu, Fox, and Brenner]{chandonia2022scope}
John-Marc Chandonia, Lindsey Guan, Shiangyi Lin, Changhua Yu, Naomi~K Fox, and Steven~E Brenner.
\newblock Scope: improvements to the structural classification of proteins--extended database to facilitate variant interpretation and machine learning.
\newblock \emph{Nucleic acids research}, 50\penalty0 (D1):\penalty0 D553--D559, 2022.

\bibitem[Chung et~al.(2022{\natexlab{a}})Chung, Kim, Mccann, Klasky, and Ye]{chung2022diffusion}
Hyungjin Chung, Jeongsol Kim, Michael~T Mccann, Marc~L Klasky, and Jong~Chul Ye.
\newblock Diffusion posterior sampling for general noisy inverse problems.
\newblock \emph{arXiv preprint arXiv:2209.14687}, 2022{\natexlab{a}}.

\bibitem[Chung et~al.(2022{\natexlab{b}})Chung, Sim, Ryu, and Ye]{chung2022improving}
Hyungjin Chung, Byeongsu Sim, Dohoon Ryu, and Jong~Chul Ye.
\newblock Improving diffusion models for inverse problems using manifold constraints.
\newblock \emph{Advances in Neural Information Processing Systems}, 35:\penalty0 25683--25696, 2022{\natexlab{b}}.

\bibitem[De~Bortoli(2022)]{de2022convergence}
Valentin De~Bortoli.
\newblock Convergence of denoising diffusion models under the manifold hypothesis.
\newblock \emph{arXiv preprint arXiv:2208.05314}, 2022.

\bibitem[De~Bortoli et~al.(2021{\natexlab{a}})De~Bortoli, Doucet, Heng, and Thornton]{de2021simulating}
Valentin De~Bortoli, Arnaud Doucet, Jeremy Heng, and James Thornton.
\newblock Simulating diffusion bridges with score matching.
\newblock \emph{arXiv preprint arXiv:2111.07243}, 2021{\natexlab{a}}.

\bibitem[De~Bortoli et~al.(2021{\natexlab{b}})De~Bortoli, Thornton, Heng, and Doucet]{de2021diffusion}
Valentin De~Bortoli, James Thornton, Jeremy Heng, and Arnaud Doucet.
\newblock Diffusion {S}chr{\"o}dinger bridge with applications to score-based generative modeling.
\newblock \emph{Advances in Neural Information Processing Systems}, 34:\penalty0 17695--17709, 2021{\natexlab{b}}.

\bibitem[Dhariwal and Nichol(2021)]{dhariwal2021diffusion}
Prafulla Dhariwal and Alexander Nichol.
\newblock Diffusion models beat gans on image synthesis.
\newblock \emph{Advances in neural information processing systems}, 34:\penalty0 8780--8794, 2021.

\bibitem[Dutordoir et~al.(2023)Dutordoir, Saul, Ghahramani, and Simpson]{dutordoir2023neural}
Vincent Dutordoir, Alan Saul, Zoubin Ghahramani, and Fergus Simpson.
\newblock Neural diffusion processes.
\newblock In \emph{International Conference on Machine Learning}, pages 8990--9012. PMLR, 2023.

\bibitem[Finzi et~al.(2023)Finzi, Boral, Wilson, Sha, and Zepeda-N{\'u}{\~n}ez]{finzi2023user}
Marc~Anton Finzi, Anudhyan Boral, Andrew~Gordon Wilson, Fei Sha, and Leonardo Zepeda-N{\'u}{\~n}ez.
\newblock User-defined event sampling and uncertainty quantification in diffusion models for physical dynamical systems.
\newblock In \emph{International Conference on Machine Learning}, pages 10136--10152. PMLR, 2023.

\bibitem[Han et~al.(2022)Han, Zheng, and Zhou]{han2022card}
Xizewen Han, Huangjie Zheng, and Mingyuan Zhou.
\newblock Card: Classification and regression diffusion models.
\newblock \emph{Advances in Neural Information Processing Systems}, 35:\penalty0 18100--18115, 2022.

\bibitem[Ho and Salimans(2022)]{ho2022classifier}
Jonathan Ho and Tim Salimans.
\newblock Classifier-free diffusion guidance.
\newblock \emph{arXiv preprint arXiv:2207.12598}, 2022.

\bibitem[Ho et~al.(2020)Ho, Jain, and Abbeel]{ho2020denoising}
Jonathan Ho, Ajay Jain, and Pieter Abbeel.
\newblock Denoising diffusion probabilistic models.
\newblock \emph{Advances in Neural Information Processing Systems}, 33:\penalty0 6840--6851, 2020.

\bibitem[Ho et~al.(2022)Ho, Salimans, Gritsenko, Chan, Norouzi, and Fleet]{ho2022Video}
Jonathan Ho, Tim Salimans, Alexey Gritsenko, William Chan, Mohammad Norouzi, and David~J. Fleet.
\newblock Video {{Diffusion Models}}, June 2022.

\bibitem[Jumper et~al.(2021)Jumper, Evans, Pritzel, Green, Figurnov, Ronneberger, Tunyasuvunakool, Bates, {\v{Z}}{\'\i}dek, Potapenko, et~al.]{jumper2021highly}
John Jumper, Richard Evans, Alexander Pritzel, Tim Green, Michael Figurnov, Olaf Ronneberger, Kathryn Tunyasuvunakool, Russ Bates, Augustin {\v{Z}}{\'\i}dek, Anna Potapenko, et~al.
\newblock Highly accurate protein structure prediction with alphafold.
\newblock \emph{Nature}, 596\penalty0 (7873):\penalty0 583--589, 2021.

\bibitem[Lin and AlQuraishi(2023)]{lin2023generating}
Yeqing Lin and Mohammed AlQuraishi.
\newblock Generating novel, designable, and diverse protein structures by equivariantly diffusing oriented residue clouds.
\newblock \emph{arXiv preprint arXiv:2301.12485}, 2023.

\bibitem[Liu et~al.(2023)Liu, Vahdat, Huang, Theodorou, Nie, and Anandkumar]{liu20232}
Guan-Horng Liu, Arash Vahdat, De-An Huang, Evangelos~A Theodorou, Weili Nie, and Anima Anandkumar.
\newblock I2 sb: Image-to-image schrodinger bridge.
\newblock \emph{arXiv preprint arXiv:2302.05872}, 2023.

\bibitem[Liu and Wu(2023)]{liu2023learning}
Xingchao Liu and Lemeng Wu.
\newblock Learning diffusion bridges on constrained domains.
\newblock In \emph{international conference on learning representations (ICLR)}, 2023.

\bibitem[Liu et~al.(2015)Liu, Luo, Wang, and Tang]{liu2015faceattributes}
Ziwei Liu, Ping Luo, Xiaogang Wang, and Xiaoou Tang.
\newblock Deep learning face attributes in the wild.
\newblock In \emph{Proceedings of International Conference on Computer Vision (ICCV)}, December 2015.

\bibitem[Lugmayr et~al.(2022)Lugmayr, Danelljan, Romero, Yu, Timofte, and Van~Gool]{lugmayr2022repaint}
Andreas Lugmayr, Martin Danelljan, Andres Romero, Fisher Yu, Radu Timofte, and Luc Van~Gool.
\newblock Repaint: Inpainting using denoising diffusion probabilistic models.
\newblock In \emph{Proceedings of the IEEE/CVF Conference on Computer Vision and Pattern Recognition}, pages 11461--11471, 2022.

\bibitem[Mathieu et~al.(2023)Mathieu, Dutordoir, Hutchinson, De~Bortoli, Teh, and Turner]{mathieu2023geometric}
Emile Mathieu, Vincent Dutordoir, Michael~J Hutchinson, Valentin De~Bortoli, Yee~Whye Teh, and Richard~E Turner.
\newblock Geometric neural diffusion processes.
\newblock \emph{arXiv preprint arXiv:2307.05431}, 2023.

\bibitem[Meng and Kabashima(2022)]{meng2022diffusion}
Xiangming Meng and Yoshiyuki Kabashima.
\newblock Diffusion model based posterior sampling for noisy linear inverse problems.
\newblock \emph{arXiv preprint arXiv:2211.12343}, 2022.

\bibitem[Nilsback and Zisserman(2008)]{nilsback08}
Maria-Elena Nilsback and Andrew Zisserman.
\newblock Automated flower classification over a large number of classes.
\newblock In \emph{Indian Conference on Computer Vision, Graphics and Image Processing}, Dec 2008.

\bibitem[Ramesh et~al.(2021)Ramesh, Pavlov, Goh, Gray, Voss, Radford, Chen, and Sutskever]{pmlr-v139-ramesh21a}
Aditya Ramesh, Mikhail Pavlov, Gabriel Goh, Scott Gray, Chelsea Voss, Alec Radford, Mark Chen, and Ilya Sutskever.
\newblock Zero-shot text-to-image generation.
\newblock In Marina Meila and Tong Zhang, editors, \emph{Proceedings of the 38th International Conference on Machine Learning}, volume 139 of \emph{Proceedings of Machine Learning Research}, pages 8821--8831. PMLR, 18--24 Jul 2021.
\newblock URL \url{https://proceedings.mlr.press/v139/ramesh21a.html}.

\bibitem[Rogers and Williams(2000)]{rogers2000diffusions}
L~Chris~G Rogers and David Williams.
\newblock \emph{Diffusions, Markov processes and martingales: Volume 2, It{\^o} calculus}, volume~2.
\newblock Cambridge university press, 2000.

\bibitem[Rozet and Louppe(2023)]{rozet2023score}
Fran{\c{c}}ois Rozet and Gilles Louppe.
\newblock Score-based data assimilation.
\newblock \emph{arXiv preprint arXiv:2306.10574}, 2023.

\bibitem[Saharia et~al.(2022)Saharia, Ho, Chan, Salimans, Fleet, and Norouzi]{saharia2022image}
Chitwan Saharia, Jonathan Ho, William Chan, Tim Salimans, David~J Fleet, and Mohammad Norouzi.
\newblock Image super-resolution via iterative refinement.
\newblock \emph{IEEE Transactions on Pattern Analysis and Machine Intelligence}, 45\penalty0 (4):\penalty0 4713--4726, 2022.

\bibitem[Simon~V et~al.(2023)Simon~V, Urszula, Mateja, and Pietro]{mathis2023nmd}
Mathis Simon~V, Julia~Komorowska Urszula, Jamnik Mateja, and Lio Pietro.
\newblock Normal mode diffusion: Towards dynamics-informed protein design.
\newblock \emph{The 2023 ICML Workshop on Computational Biology. Baltimore, Maryland, USA, 2023. C}, 2023.

\bibitem[Somnath et~al.(2023)Somnath, Pariset, Hsieh, Martinez, Krause, and Bunne]{somnath2023aligned}
Vignesh~Ram Somnath, Matteo Pariset, Ya-Ping Hsieh, Maria~Rodriguez Martinez, Andreas Krause, and Charlotte Bunne.
\newblock Aligned diffusion schr$\backslash$" odinger bridges.
\newblock \emph{arXiv preprint arXiv:2302.11419}, 2023.

\bibitem[Song et~al.(2022)Song, Vahdat, Mardani, and Kautz]{song2022pseudoinverse}
Jiaming Song, Arash Vahdat, Morteza Mardani, and Jan Kautz.
\newblock Pseudoinverse-guided diffusion models for inverse problems.
\newblock In \emph{International Conference on Learning Representations}, 2022.

\bibitem[Song et~al.(2021{\natexlab{a}})Song, Shen, Xing, and Ermon]{song2021solving}
Yang Song, Liyue Shen, Lei Xing, and Stefano Ermon.
\newblock Solving inverse problems in medical imaging with score-based generative models.
\newblock \emph{arXiv preprint arXiv:2111.08005}, 2021{\natexlab{a}}.

\bibitem[Song et~al.(2021{\natexlab{b}})Song, Sohl-Dickstein, Kingma, Kumar, Ermon, and Poole]{song2020score}
Yang Song, Jascha Sohl-Dickstein, Diederik~P Kingma, Abhishek Kumar, Stefano Ermon, and Ben Poole.
\newblock Score-based generative modeling through stochastic differential equations.
\newblock In \emph{International Conference on Learning Representations}, 2021{\natexlab{b}}.
\newblock URL \url{https://openreview.net/forum?id=PxTIG12RRHS}.

\bibitem[Song et~al.(2021{\natexlab{c}})Song, {Sohl-Dickstein}, Kingma, Kumar, Ermon, and Poole]{song2021Scorebased}
Yang Song, Jascha {Sohl-Dickstein}, Diederik~P Kingma, Abhishek Kumar, Stefano Ermon, and Ben Poole.
\newblock Score-based generative modeling through stochastic differential equations.
\newblock In \emph{International Conference on Learning Representations}, 2021{\natexlab{c}}.
\newblock URL \url{https://openreview.net/forum?id=PxTIG12RRHS}.

\bibitem[Torge et~al.(2023)Torge, Harris, Mathis, and Lio]{torge2023diffhopp}
Jos Torge, Charles Harris, Simon~V. Mathis, and Pietro Lio.
\newblock Diffhopp: A graph diffusion model for novel drug design via scaffold hopping.
\newblock \emph{arXiv preprint arXiv:2308.07416}, 2023.

\bibitem[Trippe et~al.(2022)Trippe, Yim, Tischer, Baker, Broderick, Barzilay, and Jaakkola]{trippe2022diffusion}
Brian~L Trippe, Jason Yim, Doug Tischer, David Baker, Tamara Broderick, Regina Barzilay, and Tommi Jaakkola.
\newblock Diffusion probabilistic modeling of protein backbones in 3d for the motif-scaffolding problem.
\newblock \emph{arXiv preprint arXiv:2206.04119}, 2022.

\bibitem[Vargas et~al.(2023)Vargas, Ovsianas, Fernandes, Girolami, Lawrence, and N{\"u}sken]{vargas2021bayesian}
Francisco Vargas, Andrius Ovsianas, David Fernandes, Mark Girolami, Neil~D Lawrence, and Nikolas N{\"u}sken.
\newblock Bayesian learning via neural {S}chr{\"o}dinger--{F}{\"o}llmer flows.
\newblock \emph{Statistics and Computing}, 33\penalty0 (1):\penalty0 3, 2023.

\bibitem[Watson et~al.(2023)Watson, Juergens, Bennett, Trippe, Yim, Eisenach, Ahern, Borst, Ragotte, Milles, et~al.]{watson2023novo}
Joseph~L Watson, David Juergens, Nathaniel~R Bennett, Brian~L Trippe, Jason Yim, Helen~E Eisenach, Woody Ahern, Andrew~J Borst, Robert~J Ragotte, Lukas~F Milles, et~al.
\newblock De novo design of protein structure and function with rfdiffusion.
\newblock \emph{Nature}, pages 1--3, 2023.

\bibitem[Wu et~al.(2023)Wu, Trippe, Naesseth, Blei, and Cunningham]{wu2023practical}
Luhuan Wu, Brian~L Trippe, Christian~A Naesseth, David~M Blei, and John~P Cunningham.
\newblock Practical and asymptotically exact conditional sampling in diffusion models.
\newblock \emph{arXiv preprint arXiv:2306.17775}, 2023.

\bibitem[Ye et~al.(2022)Ye, Wu, and Liu]{ye2022first}
Mao Ye, Lemeng Wu, and Qiang Liu.
\newblock First hitting diffusion models for generating manifold, graph and categorical data.
\newblock In \emph{Advances in Neural Information Processing Systems}, 2022.

\bibitem[Zhang et~al.(2018)Zhang, Isola, Efros, Shechtman, and Wang]{zhang2018perceptual}
Richard Zhang, Phillip Isola, Alexei~A Efros, Eli Shechtman, and Oliver Wang.
\newblock The unreasonable effectiveness of deep features as a perceptual metric.
\newblock In \emph{CVPR}, 2018.

\end{thebibliography}
\newpage

\onecolumn
\appendix

\section{Background on diffusion formulations}

\subsection{Continuous and discrete diffusion formulations}

The discretised DDPM versions with various discrete time schedules amount to the time-dependent OU process
\begin{align}\label{eq:OUprocess}
    \dd \rv{X}_t &= -\frac{\beta(t)}{2} \mathbf{x}_t  \dd t + \sqrt{\beta(t)} \,\fwd{\dd \rv{W}}_t 
\end{align}
where choosing different time schedules amounts to choosing different functions $\beta(t)$. This process gives rise to the Green's function for transition probabilities
\begin{align}
p(\mathbf{x},t | \mathbf{x}_0, 0) 
&= \fwd{p}_{t|0}(\mathbf{x} | \mathbf{x}_0) \\
&= \mathcal{N}\left(\mathbf{x}_0 e^{-\int_0^T \frac{\beta(s)}{2} \dd s}, \int_0^T \beta(t) e^{-\int_0^{T-t} \beta(s) \dd s} \dd t \right) \\
&= \mathcal{N} \left(\mathbf{x}_0 e^{-\int_0^T \frac{\beta(s)}{2} \dd s}, \left(1 - e^{-\int_0^{T} \beta(s) \dd s} \right) \right).
\end{align}
With $\bar{\alpha}(t) = e^{-\int_0^{T} \beta(s) \dd s}$, this is the familiar form (\cite{ho2020denoising}):
\begin{align}
p(\mathbf{x}, t | \mathbf{x}_0, 0) = \fwd{p}_{t|0}(\mathbf{x} | \mathbf{x}_0) = \mathcal{N}\left( \mathbf{x}_0 \sqrt{\bar{\alpha}(t)}, (1-\bar{\alpha}(t)) \right),
\end{align}
with $\bar{\alpha}(t)$ time-dependent and we can therefore choose different functional forms for the noise schedule by either choosing the transition parameters $\beta(t)$ or the cumulative parameters $\alpha(t)$.

If we define the noise schedule in terms of $\beta(t)$, the time-dependent OU process is immediately apparent (see \eqref{eq:OUprocess}).

If we define the noise schedule in terms of $\bar{\alpha}(t)$, the mean and variance of the corresponding OU process can simply be obtained from
\begin{align}
\beta(t) = - \frac{\dd}{\dd t}\left[ \ln{{\bar{\alpha}(t)}} \right].
\end{align}

\subsection{Score, noise and mean diffusion formulations}

The score-based model used for generation at inference time can be parametrised to model different quantities. The three most common one are the score, the noise and the mean.

When starting from the DDPM formulation of describing the diffusion process as a Gaussian linear Markov chain, it is natural to let the network predict the mean of this Gaussian, with the covariance being a fixed parameter: 

\begin{align}
\mu_{\theta} (\mathbf{x}_t, t) = \frac{1}{\sqrt{\alpha_t}} (\mathbf{x}_t - \frac{1- \alpha_t}{\sqrt{1- \bar{\alpha}(t)}} \noise_t)
\end{align}

However, we have access to the input $\mathbf{x}_t$ at training time and can therefore reparameterize the Gaussian in order to make our network predict the noise $\noise_t$instead of the mean $\mu_t$:

\begin{align}
\mathbf{x}_{t-1} = \mathcal{N}(\mathbf{x}_{t-1}; \frac{1}{\sqrt{\alpha_t}} (\mathbf{x}_t - \frac{1- \alpha_t}{\sqrt{1- \bar{\alpha}(t)}} \noise_{\theta}(\mathbf{x}_t,t))
\end{align}

When starting from the score-based SDE formulation, one can instead let the network predict the score term in order to minimise the following score matching loss:

\begin{align}
\mathcal{L} = \mathbb{E}_{t,\mathbf{x}_0, \mathbf{x}_t} || s_{\theta}(\mathbf{x}_t, t) - \nabla_{\mathbf{x}_t} \ln \fwd{p}_{t|0}(\mathbf{x}_t | \mathbf{x}_0) ||^2 / \sigma_t^2
\end{align}

    

\subsection{Doob's $h$-transform intuition}\label{app:intuit}

As mentioned before Doob's $h$-transform adds a new drift to the SDE which amounts to two terms (via Bayes Theorem), a conditional and an unconditional score:
\begin{equation}
\nabla \ln  \bwd{P}_{0|t}(\mX_0 \in B | \cdot) = \nabla \ln  \fwd{P}_{t|0}(\cdot| \mX_0 \in B) -  \nabla \ln  {P}_{t}(\cdot)
\end{equation}

Interestingly, these two terms provide for a unique intuition: the Doob's transform SDE is the time reversal of the forward SDE corresponding to \eqref{eq:back_sde}, that is the time reversal of the forward SDE
\begin{align}
\dd \bX_t &= \fwd{b}_t(\bX_t) \,\dd t + \sigma_t \fwd{ \dd \rv{W}}_t, \quad \bX_0 \sim \fwd{P}_{0}(\cdot| \mX_0 \in B)
\label{eq:for_sde},
\end{align}
coincides with the Doob transformed SDE \eqref{eq:back_sde_h} \citep{de2021simulating}. 

Thus we can view Doob's transform as the following series of steps:
\begin{enumerate}
    \item Time reverse the SDE we want to condition (\eqref{eq:back_sde_h} to \eqref{eq:for_sde}).
    \item  Impose the condition via ancestral sampling from the conditioned distribution/posterior.
    \item Time reverse once more to be in the same time direction as we started.
\end{enumerate}

\subsection{Examples}

\paragraph{Truncated normal}
Here for illustrative purposes we frame the problem of sampling from a truncated normal distribution as simulating an SDE that is given by Doob’s h-transform.

Let's remind that a 1d truncated normal distribution had a density $p(x| a, b) \propto \mathds{1}_{x \in (a, b)}(x) \mathcal{N}(x|\mu, \sigma^2)$.
Now, let's assume a data distribution $p_0(x) = \mathrm{N}(\mu, \sigma^2)$ which is noised with an OU process \eqref{eq:OUprocess}.
Thus we have that $p(x_0|x_t) = \mathrm{N}(x_0|\hat{\mu}_{0|t}(x_t), \hat{\sigma}_{0|t}(x_t)^2)$ is Gaussian, and so is $p(x_t) = \mathrm{N}(x_t|\hat{\mu}_t, \hat{\sigma}_t^2)$.
Let's add the constraint that the process hit at time $t=0$ the event $\bX_0 \in (a, b)$.
\begin{align}
    \dd \bH_t &= \beta(t) \left( \frac{\bH_t}{2} + \nabla_{\mH_t} \ln \fwd{P}_{t}(\bH_t)  - {\nabla_{\mH_t}\ln \bwd{P}_{0|t}(\mX_0 \in (a, b) \mid \mH_t) } \right)\,\dd t + \sqrt{\beta(t)} ~\bwd{ \dd \rv{W}}_t, 
\end{align}
We have that the h-transform is given by
\begin{align}
    h(t, \bH_t) 
    &= \bwd{P}_{0|t}(\mX_0 \in (a, b) | \bH_t) 
    = \int \mathds{1}_{x \in (a, b)}(\mx_0) \bwd{p}_{0| t}(\mx_0|\bH_t) \mathrm{d} \mx_0 \nonumber \\
    &= \int \mathds{1}_{x \in (a, b)}(\mx_0) \mathcal{N}(x|\hat{\mu}_{0|t}(\bH_t), \hat{\sigma}_{0|t}(\bH_t)^2) \mathrm{d} \mx_0  \nonumber \\
    &= \frac{1}{\hat{\sigma}_{0|t}(\bH_t)} \frac{\phi\left( \frac{\bH_t - \hat{\mu}_{0|t}(\bH_t)}{\hat{\sigma}_{0|t}(\bH_t)} \right)}{\Phi\left( \frac{b - \hat{\mu}_{0|t}(\bH_t)}{\hat{\sigma}_{0|t}(\bH_t)} \right) - \Phi\left( \frac{a - \hat{\mu}_{0|t}(\bH_t)}{\hat{\sigma}_{0|t}(\bH_t)} \right)}
\end{align}
where $\phi(\xi) = \frac{1}{\sqrt{2 \pi}} \exp \left( - \frac{1}{2} \xi^2 \right)$ is the pdf of a standard normal distribution, $\Phi(\xi) = \frac{1}{2}\left( 1 + \erf(\xi/\sqrt{2})\right)$ its cumulative function.
The corrective drift term due to the h-transform can then be computed via autograd.
The unconditional score term can be computed in closed form.

\section{Algorithms\label{app:algorithms}}
In this section, we reformulate multiple algorithms from the literature under our common framework as a reference for practitioners. In these algorithms, we use the following conventions: our dataset is drawn from the law $\distdata$, but we can only sample from the simpler law $\distsampling$ at inference time, which is often chosen as multivariate standard normal $\distsampling = \mathcal{N}(0,\mathbf{I})$. Therefore, we construct a forward noising process $\distdata \to \distsampling$ that is parametrised via the noise schedule $\beta_t = \beta(t), \bar{\alpha}_t = \bar{\alpha}(t)$ and try to learn the reverse denoising process $\distsampling \to \distdata$. Due to this notion of "forward", and to keep consistency with the literature on denoising diffusion models, we explicate the nomenclature $\distdata = \distO$ and $\distsampling = \distT$.

There is an additional law $\distnoise$ that is sometimes confused with $\distsampling$ since in practice both are often chosen as $\mathcal{N}(0,\mathbf{I})$, but they are two distinct laws that could in principle be different. $\distnoise$ is the law from which the noise added during the forward noising process as well as the during the reverse diffusion process is drawn from. 
\subsection{Unconditional algorithms}
\begin{algorithm}[H]
\caption{$\vert$ Unconditional training of denoising diffusion models \citep{ho2020denoising}}
\label{algo:uncond_training}
\begin{algorithmic}[1]
\Require Dataset drawn from law $\distdata = \distO $ \Comment{Dataset law $\distdata$}
\Require Noise schedule $\beta_t = \beta(t), \bar{\alpha}_t = \bar{\alpha}(t)$, parametrising process $\distdata \to \distsampling$
\Require Untrained noise predictor function $\nn(\mathbf{x}, t)$ with parameters $\theta$
\Repeat
    \State $\mathbf{x}_0 \sim \distO = \distdata$
    \State $t \sim \text{Uniform}(\{1,...,T\})$
    \Statex
    \LComment{Forward noise sample, $\mathbf{x}_t \sim \fwd{p}_{t \vert 0}(\mathbf{x}_0)$}
    \State $\noise_t \sim \mathcal{P}_\text{noise}$
    \Comment{Often Brownian motion, $\distnoise = \mathcal{N}(0,\mathbf{I})$}
    \State $\mathbf{x}_t \gets \sqrt{\bar\alpha_t} \mathbf{x}_0 + \sqrt{1 - \bar\alpha_t} \noise_t$ 

    \Statex
    \LComment{Estimate noise of noised sample}
    \State $\l{\e{\noise}} \gets \nn(\mathbf{x}_t, t)$ 

    \Statex
    \State \text{Take gradient descent step on} 
    \Statex $\quad \nabla_{\theta}L(\noise_t, \e{\noise}_\theta)$
    \Comment{Typically, loss $L(x_\text{true}, x_\text{pred}) = || x_\text{true} - x_\text{pred} ||^2$}
\Until converged or max epoch reached

\end{algorithmic}
\end{algorithm}
\vspace{-0.5cm}
\begin{algorithm}[H]
\caption{$\vert$ Unconditional sampling with denoising diffusion models \citep{ho2020denoising} }
\label{algo:uncond_sampling}
\begin{algorithmic}[1]
    \Require Unconditionally trained noise predictor $\nn(\mathbf{x}_t, t)$ 
    \Require Noise schedule $\beta_t = \beta(t), \bar{\alpha}_t = \bar{\alpha}(t)$, parametrising process $\distdata \to \distsampling$
    \LComment{Sample a starting point $\mathbf{x}_T$}
    \State $\mathbf{x}_T \sim \distT = \distsampling$
    \Comment{Often $\mathcal{P}_T = \mathcal{N}(0,\mathbf{I})$}
    
    \Statex
    \LComment{Iteratively denoise for $T$ steps}
    \For{$t$ in $(T, T-1, \dots, 1)$}
        \LComment{Predict noise with learned network}
        \State $\l{\e{\noise}} = \nn(\mathbf{x}_t, t)$

        \Statex
        \LComment{Denoise sample with learned reverse process $\mathbf{x}_{t-1} \sim \bwd{p}_{t-1|t}(\mathbf{x}_t)$}
        \LComment{Perform reverse drift}
        \State $\mathbf{x}_{t-1} \gets \dfrac{1}{\sqrt{1-\beta_t}} \left(
        \mathbf{x}_t - \dfrac{\beta_t}{\sqrt{1 - \bar\alpha_t}} \l{\e{\noise}} \right)$

        \Statex
        \Statex
        \LComment{Perform reverse diffusion, which is often Brownian motion in $\mathbb{R}^n$, i.e. $\distnoise = \mathcal{N}(0,\mathbf{I})$}
        \State $\noise_t \sim \distnoise$ if $t>1$ else $\noise_t \gets 0$
        \State $\mathbf{x}_{t-1} \gets \mathbf{x}_{t-1} + \sigma_t \noise_t$ 
        \Comment{A common choice is $\sigma_t = \beta(t)$}
    \EndFor
    \State \Return $\mathbf{x}_0$
\end{algorithmic}
\end{algorithm}
\subsection{Conditional training\label{app:cond_training_algos}}\vspace{-0.5cm}
\begin{algorithm}[H]
\caption{$\vert$ Classifier-free conditional training \citep{ho2022classifier}}
\label{algo:classifier_free_conditional_training}
\begin{algorithmic}[1]
\Require Dataset drawn from $\distdata$ \Comment{Dataset law $\distdata$ over data and auxiliary variable}
\Require Noise schedule $\beta_t = \beta(t), \bar{\alpha}_t = \bar{\alpha}(t)$, parametrising process $\distdata \to \distsampling$
\Require Untrained noise predictor function $\nn(\mathbf{x}, t)$ with parameters $\theta$
\Repeat
    \State $\mathbf{x}_0, \my \sim \distO = \distdata$
    \State $\noise_t \sim \mathcal{P}_\text{noise}$
    \Comment{Often Brownian motion, $\distnoise = \mathcal{N}(0,\mathbf{I})$}
    \State $t \sim \text{Uniform}(\{1,...,T\})$
    \State $\mathbf{x}_t = \sqrt{\bar\alpha_t \mathbf{x}_0} + \sqrt{1 - \bar\alpha_t} \noise_t$
    \State $\l{\e{\noise}} = \nn(\mathbf{x}_t, t, \my)$
    \State \text{Take gradient descent step on} 
    \Statex $\quad \nabla_{\theta}L(\noise_t, \e{\noise}_\theta)$
    \Comment{Typically, $L(x_\text{true}, x_\text{pred}) = || x_\text{true} - x_\text{pred} ||^2$}
\Until converged or max epoch reached
\end{algorithmic}
\end{algorithm}\vspace{-0.7cm}
\begin{algorithm}[H]
\caption{$\vert$ RFDiffusion conditional training \citep{watson2023novo} }
\label{algo:rfdiff}
\begin{algorithmic}[1]
\Require Dataset drawn from $\distdata$ \Comment{Dataset law $\distdata$}
\Require Noise schedule $\beta_t = \beta(t), \bar{\alpha}_t = \bar{\alpha}(t)$, parametrising process $\distdata \to \distsampling$
\Require Untrained noise predictor function $\nn(\mathbf{x}, t$\hl{$, M)$} with parameters $\theta$
\Repeat
    \State $\mathbf{x}_0 \sim \distO = \distdata$
    \State $t \sim \text{Uniform}(\{1,...,T\})$

    \BeginBox[fill=shadecolor]
    \State $\mathbf{x}_0^{[M]} \cup \mathbf{x}_0^{[\backslash M]} \gets \mathbf{x}_0$  \Comment{Randomly partition data point into motif and rest}

    \Statex
    \LComment{Forward noise the non-motif rest via sampling from $\fwd{p}_{0 \vert t}(\mathbf{x}_0)$}
    \State $\noise_t \sim \mathcal{P}_\text{noise}$
    \State $\mathbf{x}_t^{[\backslash M]} \gets \sqrt{\bar\alpha_t} \mathbf{x}_0^{[\backslash M]} + \sqrt{1 - \bar\alpha_t} \noise_t^{[\backslash M]}$ 

    \Statex
    \LComment{Combine unnoised motif with noised rest and set timestep of motif part to 0}
    \State $\mathbf{x}_t \gets \mathbf{x}_0^{[M]} \cup \mathbf{x}_t^{[\backslash M]}$
    \State $t^{[M]} \gets 0$
    \EndBox
    \State $\l{\e{\noise}} \gets \nn(\mathbf{x_t}, t$\hl{$, M)$} \Comment{Estimate noise of sample with noised rest}
    \State \text{Take gradient descent step on} 
    \Statex $\quad \nabla_{\theta}L(\noise, \e{\noise}_\theta)$
    \Comment{Typically, $L(x_\text{true}, x_\text{pred}) = || x_\text{true} - x_\text{pred} ||^2$}
\Until converged or max epoch reached

\end{algorithmic}
\end{algorithm}\vspace{-0.6cm}
\begin{algorithm}[H]
\caption{$\vert$ Amortised training -- i.e. Doob's $h$-transform conditional training \new }
\label{algo:cond_dobsh}
\begin{algorithmic}[1]
\Require Dataset drawn from $\distdata$ \Comment{Dataset law $\distdata$}
\Require Noise schedule $\beta_t = \beta(t), \bar{\alpha}_t = \bar{\alpha}(t)$, parametrising process $\distdata \to \distsampling$
\Require Untrained noise predictor function $\nn(\mathbf{x}, t$\hl{$,\mathbf{x}^{[M]}, M)$} with parameters $\theta$
\Repeat
    \State $\mathbf{x}_0 \sim \distO = \distdata$
    \State $t \sim \text{Uniform}(\{1,...,T\})$

    \BeginBox[fill=shadecolor]
    \State $\mathbf{x}_0^{[M]} \cup \mathbf{x}_0^{[\backslash M]} \gets \mathbf{x}_0$  \Comment{Randomly partition data point into motif and rest}
    \EndBox
    
    \LComment{Forward noise full sample via sampling from $\fwd{p}_{0 \vert t}(\mathbf{x}_0)$}
    \State $\noise_t \sim \mathcal{P}_\text{noise}$
    \State $\mathbf{x}_t \gets \sqrt{\bar\alpha_t} \mathbf{x}_0 + \sqrt{1 - \bar\alpha_t} \noise_t$ 

    \Statex
    \LComment{Estimate noise of sample with original motif as additional input}
    \State $\l{\e{\noise}} \gets \nn(\mathbf{x}_t, t$\hl{$,\mathbf{x}_0^{[M]}, M)$}
    \State \text{Take gradient descent step on} 
    \Statex $\quad \nabla_{\theta}L(\noise, \e{\noise}_\theta)$
    \Comment{Typically, $L(x_\text{true}, x_\text{pred}) = || x_\text{true} - x_\text{pred} ||^2$}
\Until converged or max epoch reached

\end{algorithmic}
\end{algorithm}


\subsection{Conditional sampling\label{app:cond_sampling_algos}}
\begin{algorithm}[H]
\caption{$\vert$ RFDiffusion conditional sampling \citep{watson2023novo} }
\label{algo:rf_diffusion}
\begin{algorithmic}[1]
    \Require \hl{Conditionally trained} noise predictor $\nn(\mathbf{x}, t$\hl{$, M)$}
    \Require Target motif/context $\mathbf{x}_0^{[M]}$
    \Require Noise schedule $\beta_t = \beta(t), \bar{\alpha}_t = \bar{\alpha}(t)$, parametrising process $\distdata \to \distsampling$
    \LComment{Sample a starting point $\mathbf{x}_T$}
    \State $\mathbf{x}_T \sim \distT = \distsampling$
    
    \Statex
    \LComment{Iteratively denoise for $T$ steps}
    \Comment{Often $\mathcal{P}_T = \mathcal{N}(0,\mathbf{I})$}
    \For{$t$ in $(T, T-1, \dots, 1)$}
        \BeginBox[fill=shadecolor]
        \LComment{Overwrite motif variables with target motif and reset their time parameter}
        \LComment{Note: Original RFDiffusion zero-centers $\mathbf{x}_t$ and $\mathbf{x}_0^{[M]}$ individually for equivariance.}
        \State $\mathbf{x}_{t}^{[M]} \gets \mathbf{x}_0^{[M]}$ \Comment{Set noisy motif to unnoised motif}
        \State $t^{[M]} \gets 0$ \Comment{Set timesteps for motif to 0}
        \EndBox
        \State $\l{\e{\noise}} = \nn(\mathbf{x}_t, t$\hl{$, M)$} \Comment{Predict noise with learned network}

        \Statex
        \LComment{Denoise sample with learned reverse process $\mathbf{x}_{t-1} \sim \bwd{p}_{t-1|t}(\mathbf{x}_t)$}
        \LComment{Perform reverse drift}
        \State $\mathbf{x}_{t-1} \gets \dfrac{1}{\sqrt{1-\beta_t}} \left(
        \mathbf{x}_t - \dfrac{\beta_t}{\sqrt{1 - \bar\alpha_t}} \l{\e{\noise}} \right)$
        \Statex
        \Statex
        \LComment{Perform reverse diffusion, which is often Brownian motion in $\mathbb{R}^n$, i.e. $\distnoise = \mathcal{N}(0,\mathbf{I})$}
        \State $\noise_t \sim \distnoise$ if $t>1$ else $\noise_t \gets 0$
        \State $\mathbf{x}_{t-1} \gets \mathbf{x}_{t-1} + \sigma_t \noise_t$ 
        \Comment{A common choice is $\sigma_t = \beta(t)$}
    \EndFor
    \State \Return $\mathbf{x}_0$
\end{algorithmic}
\end{algorithm}\vspace{-0.6cm}
\begin{algorithm}[H]
\caption{$\vert$ Replacement conditional sampling }
\label{algo:replace}
\begin{algorithmic}[1]
    \Require Unconditionally trained noise predictor $\nn(\mathbf{x}_t, t)$ 
    \Require Noise schedule $\beta_t = \beta(t), \bar{\alpha}_t = \bar{\alpha}(t)$, parametrising process $\distdata \to \distsampling$
    \Require Target motif $\mathbf{x}_0^{[M]}$
    \LComment{Sample a starting point $\mathbf{x}_T$}
    \State $\mathbf{x}_T \sim \distT = \distsampling$
    
    \Statex
    \LComment{Iteratively denoise for $T$ steps}
    \Comment{Often $\mathcal{P}_T = \mathcal{N}(0,\mathbf{I})$}
    \For{$t$ in $(T, T-1, \dots, 1)$}         
        \LComment{Predict noise with learned network}
        \State $\l{\e{\noise}} \gets \nn(\mathbf{x}_t, t)$

        \Statex
        \LComment{Denoise sample with learned reverse process $\mathbf{x}_{t-1} \sim \bwd{p}_{t-1|t}(\mathbf{x}_t)$}
        \LComment{Perform reverse drift}
        \State $\mathbf{x}_{t-1} \gets \dfrac{1}{\sqrt{1-\beta_t}} \left(
        \mathbf{x}_t - \dfrac{\beta_t}{\sqrt{1 - \bar\alpha_t}} \l{\e{\noise}} \right)$

        \Statex
        \Statex
        \LComment{Perform reverse diffusion, which is often Brownian motion in $\mathbb{R}^n$, i.e. $\distnoise = \mathcal{N}(0,\mathbf{I})$}
        \State $\noise_t \sim \distnoise$ if $t>1$ else $\noise_t \gets 0$
        \State $\mathbf{x}_{t-1} \gets \mathbf{x}_{t-1} + \sigma_t \noise_t$ 
        \Comment{A common choice is $\sigma_t = \beta(t)$}

        \BeginBox[fill=shadecolor]
        \Statex
        \LComment{Forward noise the target motif $\mathbf{x}_{t-1}^{[M]} \sim \fwd{p}_{0 \vert {t-1}}(\mathbf{x}_0^{[M]})$}
        \State $\bm\eta_{t-1} \sim \distnoise$ if $t > 1$ else $\bm\eta_{t-1} \gets 0$
        \State $\mathbf{x}_{t-1}^{[M]} \gets \sqrt{\bar\alpha_{t-1}} \mathbf{x}_0^{[M]} + \sqrt{1 - \bar\alpha_{t-1}} \bm\eta_{t-1}$ 

        \State $\mathbf{x}_{t-1} \gets \mathbf{x}_{t-1}^{[\backslash M]} \cup \mathbf{x}_{t-1}^{[M]} $  \Comment{Insert noised motif into current sample}
        \EndBox
    \EndFor
    \State \Return $\mathbf{x}_0$
\end{algorithmic}
\end{algorithm}

\begin{algorithm}[H]
\caption{$\vert$ Reconstruction Guidance (i.e. Moment Matching (MM) Approximation to $h$-transform)}
\label{algo:uncond_Doob'sh}
\begin{algorithmic}[1]
    \Require Unconditionally trained noise predictor $\nn(\mathbf{x}_t, t)$ ,  target motif/context $\mathbf{x}_0^{[M]}$.
    \Require Noise schedule $\beta_t = \beta(t), \bar{\alpha}_t = \bar{\alpha}(t)$, parameterising process $\distdata \to \distsampling$
    \BeginBox[fill=shadecolor]
    \Require Guidance scale (schedule) $\gamma_t = \gamma(t)$
    \Require Conditioning loss $l(x_\text{true}, x_\text{pred})$. e.g,  Gaussian MM $l(x_\text{true}, x_\text{pred}) = ||x_\text{true} - x_\text{pred}||^2$
    \EndBox
    \LComment{Sample a starting point $\mathbf{x}_T$}
    \State $\mathbf{x}_T \sim \distT = \distsampling$
    \Comment{Often $\mathcal{P}_T = \mathcal{N}(0,\mathbf{I})$}
    
    \Statex
    \LComment{Iteratively denoise and condition for $T$ steps}
    \For{$t$ in $(T, T-1, \dots, 1)$}
        \State $\l{\e{\noise}} = \nn(\mathbf{x}_t, t)$ \Comment{Predict noise with learned network}

        \BeginBox[fill=shadecolor]
        \Statex
        \LComment{Estimate current denoised estimate via Tweedie's formula}
        \State $\e{\mathbf{x}}_0(\mathbf{x}_t, \l{\e{\noise}}) \gets \frac{1}{\sqrt{\bar{\alpha}_t}}(\mathbf{x}_t - \sqrt{1-\bar{\alpha}_t} \l{\e{\noise}}) $
        \Comment{c.f. also eq.~15 in \cite{ho2020denoising}}
        \Statex
        \LComment{Perform gradient descent step towards condition on motif dimensions $M$}
        \State $\mathbf{x}_t \gets \mathbf{x}_t - \gamma_t \nabla_x l(\mathbf{x}_0^{[M]}, \e{\mathbf{x}}_0^{[M]}(\mathbf{x}_t, \l{\e{\noise}}))$
        \Comment{Requires backprop through \nn via e.g. $L_2$ loss}
        \EndBox

        \Statex
        \LComment{Denoise sample with learned reverse process $\mathbf{x}_{t-1} \sim \bwd{p}_{t-1|t}(\mathbf{x}_t)$}
        \State $\mathbf{x}_{t-1} \gets (1-\beta_t)^{-1/2} \left(
        \mathbf{x}_t - {\beta_t}{({1 - \bar\alpha_t})^{-1/2}} \l{\e{\noise}} \right)$ \Comment{Perform reverse drift}

        \LComment{Perform reverse diffusion, which is often Brownian motion in $\mathbb{R}^n$, i.e. $\distnoise = \mathcal{N}(0,\mathbf{I})$}
        \State $\noise_t \sim \distnoise$ if $t>1$ else $\noise_t \gets 0$
        \State $\mathbf{x}_{t-1} \gets \mathbf{x}_{t-1} + \sigma_t \noise_t$ 
        \Comment{A common choice is $\sigma_t = \beta(t)$}
    \EndFor
    \State \Return $\mathbf{x}_0$

\end{algorithmic}
\end{algorithm}\vspace{-0.6cm}
\begin{algorithm}[H]
\caption{$\vert$ Replacement conditional Sampling \citep{lugmayr2022repaint} }
\label{algo:repaint}
\begin{algorithmic}[1]
    \Require Unconditionally trained noise predictor $\nn(\mathbf{x}_t, t)$ 
    \Require Noise schedule $\beta_t = \beta(t), \bar{\alpha}_t = \bar{\alpha}(t)$, parametrising process $\distdata \to \distsampling$
    \Require Target motif $\mathbf{x}_0^{[M]}$
    \LComment{Sample a starting point $\mathbf{x}_T$}
    \State $\mathbf{x}_T \sim \distT = \distsampling$
    
    \Statex
    \LComment{Iteratively denoise for $T$ steps}
    \Comment{Often $\mathcal{P}_T = \mathcal{N}(0,\mathbf{I})$}
    \For{$t$ in $(T, T-1, \dots, 1)$}
    \Comment{$T$ time steps}
        \BeginBox[fill=shadecolor]
        \For{$r$ in $1, \dots, R$}
        \Comment{$R$ repaint steps}
        \EndBox
            \LComment{Predict noise with learned network}
            \State $\l{\e{\noise}} \gets \nn(\mathbf{x}_t, t)$

            \Statex
            \LComment{Denoise sample with learned reverse process $\mathbf{x}_{t-1} \sim \bwd{p}_{t-1|t}(\mathbf{x}_t)$}
            \State $\mathbf{x}_{t-1} \gets (1-\beta_t)^{-1/2} \left(
        \mathbf{x}_t - {\beta_t}{({1 - \bar\alpha_t})^{-1/2}} \l{\e{\noise}} \right)$  \Comment{Perform reverse drift}
    
            \LComment{Perform reverse diffusion, often Brownian motion in $\mathbb{R}^n$, i.e. $\distnoise = \mathcal{N}(0,\mathbf{I})$}
            \State $\noise_t \sim \distnoise$ if $t>1$ else $\noise_t \gets 0$
            \State $\mathbf{x}_{t-1} \gets \mathbf{x}_{t-1} + \sigma_t \noise_t$ 
            \Comment{A common choice is $\sigma_t = \beta(t)$}

            \BeginBox[fill=shadecolor]
            \Statex
            \LComment{Forward noise the target motif $\mathbf{x}_{t-1}^{[M]} \sim \fwd{p}_{0 \vert {t-1}}(\mathbf{x}_0^{[M]})$}
            \State $\bm\eta_{t-1} \sim \distnoise$ if $t > 1$ else $\bm\eta_{t-1} \gets 0$
            \State $\mathbf{x}_{t-1}^{[M]} \gets \sqrt{\bar\alpha_{t-1}} \mathbf{x}_0^{[M]} + \sqrt{1 - \bar\alpha_{t-1}} \bm\eta_{t-1}$ 

            \State $\mathbf{x}_{t-1} \gets \mathbf{x}_{t-1}^{[\backslash M]} \cup \mathbf{x}_{t-1}^{[M]} $  \Comment{Insert noised motif into current sample}

            \If{$r < R$ and $t > 1$} \Comment{Forward noise sample from $t-1$ to $t$, $\mathbf{x}_t \sim \fwd{p}_{t|t-1}(\mathbf{x}_{t-1})$}
                \State $\bm\zeta_{t-1} \sim \distnoise$
                \State $\mathbf{x}_t \gets \sqrt{1-\beta_{t-1}} \mathbf{x}_{t-1} + \sqrt{\beta_{t-1}} \bm\zeta_{t-1}$ 
            \EndIf
        \EndFor
        \EndBox
    \EndFor
    \State \Return $\mathbf{x}_0$

\end{algorithmic}
\end{algorithm}


\section{Amortised learning of Doob's transform}

\subsection{Proof of Proposition \ref{prop:train}}

\begin{proof}
Via the mean squared error property of the conditional expectation the minimiser is given by:
\begin{align}
    f^*_t(\vh, \m, \Pm) &= \E \left[ \nabla_{\fX_t}\ln \fwd{p}_{t|0}(\fX_t|\fX_0) | \mY = \m, \fX_t =\vh , \Pm = \A \right]
\end{align}
Then:
\begin{align}
    f&^*_t(\vh, \m, \Pm) = \int \nabla_{\vh}\ln \fwd{p}_{t|0}(\vh|\fX_0 ) \bwd{p}_{0|t}( \fX_0 | \fX_t =\vh, \mY=\m, , \Pm = \A ) \d\fX_0 \nonumber \\
    &=  \int \frac{\nabla_{\vh} \fwd{p}_{t|0}(\vh| \fX_0 )}{ \fwd{p}_{t|0}(\vh| \fX_0 )}  \frac{\bwd{p}_{t|0}( \fX_t=\vh | \fX_0, \mY=\m) p(\fX_0| \mY=\m,, \Pm = \A )}{p(\fX_t =\vh | \mY=\m, , \Pm = \A )} \d\fX_0\nonumber \\
    &= \frac{1}{p(\fX_t =\vh | \mY=\m, , \Pm = \A )}\int \frac{\nabla_{\vh} \fwd{p}_{t|0}(\vh| \fX_0 )} {\fwd{p}_{t|0}(\vh| \fX_0 )} {\bwd{p}_{t|0}( \fX_t =\vh | \fX_0) p(\fX_0| \mY=\m, , \Pm = \A )} \d \fX_0 \nonumber \\
    &= \frac{1}{p(\fX_t =\vh | \mY=\m, , \Pm = \A )}\nabla_{\vh} \int {\fwd{p}_{t|0}(\vh| \fX_0 )}  { p(\fX_0| \mY=\m, , \Pm = \A )} \d \fX_0 \  \nonumber \\
    &= \frac{1}{\fwd{p}( \fX_t =\vh | \mY = \m, , \Pm = \A )}\nabla_{\vh}\fwd{p}( \fX_t =\vh | \Pm \fX_0 = \m) \\
    &= \nabla_{\vh}\ln\fwd{p}( \fX_t =\vh | \mY = \m, , \Pm = \A ) \nonumber,
\end{align}
\end{proof}

\section{Related Work Discussion}

\paragraph{Classifier free guidance}{

As we highlighted before, methodologies such as classifier free guidance \citep{ho2022classifier} do not model the measurement operator explicitly. As a result, if these methods are applied to settings such as motif-scaffolding or image out-painting (where the conditioning is on a subset of the random variable), these methodologies would only denoise the scaffolding and the missing image patches. This is different to our approach which adds noise to both motif and scaffolding and then proceeds to denoise both jointly as part of the same space. In a way, one can view RFDiffusion's conditional training as an application of classifier-free guidance to this subset conditioning setting.
}

\paragraph{Image 2 Image Schr\"{o}dinger Bridges (I2SB \cite{liu20232})}{ I2SB  and more generally aligned Schr\"{o}dinger Bridges \citep{somnath2023aligned} are a recently proposed class of conditional generative models based on ideas from Schr\"{o}dinger bridges.

The premise of these methods is that they aim to learn an interpolating diffusion between a clean data sample and a corrupted / altered-corrupted data sample. This is in contrast to our framework/approaches: we consider an unconditioned SDE and condition it to hit an event at a particular time, thus learning an interpolating distribution between noise and an un-corrupted target distribution. This results in several algorithmic differences:

\begin{itemize}
    \item At its core, I2SB treat $\mY = \Pm(\mX_0) +\eta$ and $\mX_0$ as source and target distributions respectively; thus, at inference/sampling time, $\mY$ is provided to the learned SDE which generates approximate samples from $\law {\mX_0}$. However, in our approach, the source distribution is $\gN(0,I)$ and we pass $\mY$ to the score network to then obtain approximate samples from $\law {\mX_0}$.
    
    \item The score network in I2SB is a function only of $\mX_t$ and not $\mY=\Pm(\mX_0) +\eta$. This means that in I2SB, the network is parametrised as $\epsilon_{\theta}(t,\mX_t)$, whilst in our setting we parametrise as $\epsilon_{\theta}(t,\mX_t, \Pm(\mX_0) +\eta, \Pm)$. In the case of completion tasks like motif-scaffolding or image out-painting, our paramerisation looks something like $\epsilon_{\theta}(t,\mX_t, \mX_0^{\mathrm{mask}}, \mathrm{mask})$. This makes the task much easier for the network as we effectively provide it with a binary variable indicating which parts of the image are conditioned and which are not. In I2SB,  the network must learn this on its own. Furthermore, as we show in \cref{prop:train}, adding this to the network parametrisation is essential to allow recovering the true conditional score.
    
    \item The training procedure in I2SB uses the diffusion bridge $p(\mX_t| \mX_0, \Pm(\mX_0) +\eta)$ to add noise to both the source and target distributions, whilst our forward process is given by the transition density of an OU process $p(\mX_t| \mX_0)$ and is identical to standard DDPM/VP-SDE \citep{song2021Scorebased,ho2020denoising} noise adding procedures.
\end{itemize}

To summarise: whilst both methodologies employ similar mathematical methodologies (e.g. Diffusion Bridges \citep{de2021simulating}), their ideations and resulting methods are fundamentally different: on one side, \cite{liu20232} learns an interpolating distribution between the unconditioned $p(\mX_0)$ and conditioned $p(\mY| \mX_0)$ samples. On the other, we learn a denoising procedure that directly samples from the posterior  $p( \mX_0 | \mY)$; via this, we derive and explain most popular approaches for conditioning denoising diffusion models as part of our framework.
}

\paragraph{CDE}{ CDE \citep{batzolis2021conditional} is the adaptation of  \citep{ho2022classifier,saharia2022image} to inverse problem-like settings, deriving a variation of classifier-free guidance to a measurement model styled scenario. Whilst they do not focus on the measurement model, they estimate a very similar quantity as our Proposition 2.5  
\begin{align}
   f^{\mathrm{CDE}} (\vh, \m)=  \nabla_{\vh} \log \fwd{p}_{t|0}(\vh \vert \mY = \m )
\end{align}
In contrast to to the amortised conditional training:
\begin{align}
   f^{\mathrm{Doobs}} (\vh, \m, \A)=  \nabla_{\vh} \log \fwd{p}_{t|0}(\vh \vert \mY = \m , \Pm=\A)
\end{align}
when explicitly considering the distribution over the measurement model, one can see that the quantities are related to one another via marginalizing the measurement model $p_{\Pm}$. This introduces several practical and conceptual differences:
\begin{itemize}
    \item If we consider in/out painting as an example, the score network estimating $f^{\mathrm{CDE}}$ is not explicitly aware of where in the image the missing pixels are. As a result, it must perform inference over $\Pm$ (effectively marginalizing it) in order to know where to complete the image. This is clearly a much harder task for a single network to learn than conditioning on $\Pm$ where we provide this information.
    
    \item Viewed under the lens of the h-transform, $f^{\mathrm{CDE}}$ can be viewed as amortising the event $\Pm(\fX_0) = \m$ for random $\Pm$. It therefore falls under the soft constraint settings described in Sec.~\ref{sec:soft_guidance} since $\Pm(\fX_0) | \fX_0$ is not a delta. Our quantity $f^{\mathrm{Doobs}}$ is amortising over $\A(\fX_0) = \m$ for deterministic $\A$ and is therefore part of the more classical hard constraint domain of Doobs transform (Sec.~\ref{sec:hard_guidance}). We believe amortising over these simpler deterministic events can offer an advantage in making the problem easier to learn.
\end{itemize}

Finally, notice that unlike our fine-tuned h-transform objective CDE seeks to learn the conditional score model from scratch rather than fine-tune an existing one.
}

\paragraph{First Hitting Diffusions}{
A line of generative modelling methods proposed in \citep{ye2022first,liu2023learning} utilise the h-transform for unconditional generative modelling in the following settings: 

\begin{itemize}
    \item Hitting the target distribution $p_{\mathrm{data}}$ in a finite amount of time $[0,T]$ via time reversing an h-transformed VP-SDE conditioned to hit $0$ at time $T$.
    \item Constraining a diffusion process at time $T$ to lie in a subset of the reals $\Omega \subseteq \sR^d$. 
\end{itemize}

Whilst the aforementioned work uses a similar methodology and theory the focus is more in line with unconditional generative modelling rather than our setting which seeks to sample from the posterior arising in an inverse model.
}

\paragraph{RFDiffusion}{

As highlighted in \cref{algo:rfdiff} and in contrast to our approach, RFDiffusion \citep{watson2023novo} does not noise the motif coordinates $\fX_0^{[M]}$ with the forward OU-Process, instead it directly aims to sample from $p(\fX_t^{[\backslash M]} | \fX_0^{[ M]})$ and estimate this score while keeping the motif fixed. 

We can relate this approach to our amortised learning of Doob's $h$-transform, by noting that RF diffusion can be understood as learning the marginal conditional score:
\begin{align}
    p(\fX_t^{[\backslash M]} | \fX_0^{[ M]})   = \int   \overbrace{p(\fX_t | \fX_0^{[ M]}) }^{\propto h(t,\fX_t)p_t(\fX_t) } d \fX_t^{[ M]}.
\end{align}
This can be viewed as RFDiffusion estimating a marginal counterpart of our amortised $h$-transform approach.
See \cref{algo:cond_dobsh,algo:rfdiff} for more details on how these approaches differ in a pseudo-code implementation.
}

\section{Experimental details: image experiments}
\label{app:experimental_details_images}

In the image experiment, we use the DDPM~\citep{ho2020denoising} formulation for the diffusion model with $N=1000$ steps, a linear $\beta$-schedule with $\beta_0=10^{-4}$ and $\beta_N=2\cdot10^{-2}$.

\paragraph{Data}
We focus on the {\scshape CelebA}~\citep{liu2015faceattributes} and {\scshape Flowers}~\citep{nilsback08} image datasets.
For each of these datasets, we follow the same preprocessing procedure consisting of centrally cropping the image to size $64\times 64$, and rescaling to pixel values $[-1,1]$.
We use this information to also clip our model's prediction.

\paragraph{Noise model}
The noise model $\epsilon_\theta$ consists of a UNET architecture with four downsampling blocks consisting of 2d convolutional layers of dimensionality 128, 256, 384 and 512, respectively. We apply attention in the middle layers of the UNet with four heads.
Throughout the network, we use the SiLU activation function, no dropout and group normalisation layers.
The amortised network differs from the unconditional network in the fact that it accepts as input twice the number of channels (six instead of only three RGB channels).
The unconditional models operate directly on the three RGB channels while the amortised network operates on the RBG channels, the mask and the condition.
We can represent the mask and the condition information, however, into a single input with the same dimension as the image. The values of this input will be equal to the condition when the mask is $1$ and set to a padding value of $-2$ where the mask is $0$.
We concatenate the image $\R^{3\times H \times W}$ with the condition and mask input of size $\R^{3\times H \times W}$ into an image with six channels.
Due to this minor difference, our amortised network has $68.159$M parameters while the unconditional networks have $68.156$M parameters (roughly 0.005\% fewer).

\paragraph{Methods}
In the amortised setting we follow \cref{algo:cond_dobsh}. In 90\% of the training steps, we pass a condition to the network. The other 10\% contains a mask consisting of only 0's.
For the reconstruction guidance method, we use a guidance term of $\gamma=10.0$.

\paragraph{Metrics}
We measure the performance of the methods using mean squared error (MSE) and the perceptual metric LPIPS.
Both these metrics compare the similarity between the original image (from which a patch was taken) and the conditional sample.
For each metric, we compute the mean across 64 test images and repeat the experiment 5 times to get error estimates.

\section{H-transform under soft constraints} \label{sec:genh}

%
\subsection{Proof of Proposition \ref{prop:noisy}}
\begin{proof}
Consider the following reference SDE starting at the posterior of interest:
\begin{align} 
    \dd \bX_t &= f_t( \bX_t )\,\dd t + \sigma_t\;\ddf \bW_t, \quad \bX_0 \sim  \frac{p(\vy|\vx_0)\textcolor{magenta}{p_{\mathrm{data}}(\vx_0)}}{p(\vy)}
\end{align}

Now let us use $p_{t|y}(\vx|\vy) = \int p(\vx_t |\vx_0) \dd p(\vx_0|\vy)$ to denote the marginal of the above SDE and as before $p_t$ to denote the marginal of the reference starting at the data distribution. Then it follows that

\begin{align}
    p_{t|y}(\vx |\vy) &= p_t(\vx)p(\vy)^{-1} \int \textcolor{magenta}{\frac{\fwd{p}_{t|0}(\vx| \vx_0)}{p_t(\vx)}} p(\vy |\vx_0) \textcolor{magenta}{p_{\mathrm{data}}(\vx_0)} d\vx_0 \label{eq:post1}\\
    &=  p_t(\vx)p(\vy)^{-1} \int {\textcolor{magenta}{\bwd{p}_{0|t}(\vx_0| \vx)}} p(\vy |\vx_0)  d\vx_0 \label{eq:post2}
\end{align}
and thus the score of the reference starting at the posterior is given by:
\begin{align}
   \nabla_\vx \ln  p_{t|y}(\vx|\vy) = \nabla_\vx \ln  p_t(\vx) + \nabla_\vx \ln \int \bwd{p}_{0|t}(\vx_0| \vx) p(\vy |\vx_0)  d\vx_0  
\end{align}
Note that this remark highlights that the score used in DPS \citep{chung2022diffusion} (i.e. $\nabla_\vx \ln  p_{t|y}(\vx |\vy)$) is in fact the score of an OU process starting at $p(\vx_0|\vy)$ notice the cancellation going from Equations \ref{eq:post1} to \ref{eq:post2} was only possible since the \textcolor{magenta}{prior} in our target posterior is the initial distribution for the forward SDE (in their case an OU-process) with marginal $p_t$, these considerations are subtle yet important and omitted in prior works. Whilst this is akin the relationship motivated in DPS as $\nabla \ln p_{t|y}(\vx|\vy) = \nabla \ln p_t(\vx) + \nabla \ln p_t(\vy|\vx) $, DPS fails to convey that this is in fact the score of a VP-SDE with the posterior $p(\vx_0|\vy)$ as its initial distribution.

\end{proof}

\end{document}